%% file: capture-walking.tex
\documentclass{IEEEtran}
\usepackage{algorithm, algorithmic}
\usepackage{amssymb, amsmath, amsthm}
\usepackage{bm}
\usepackage{color}
\usepackage[dvipsnames]{xcolor}
\usepackage{enumerate}
\usepackage{graphicx}
\usepackage[colorlinks]{hyperref}
\usepackage{soul}
\usepackage[caption=false, font=footnotesize]{subfig}
\usepackage{url}

\usepackage[draft]{pdfcomment}

\graphicspath{{figures/}}

\newtheorem{definition}{Definition}
\newtheorem{property}{Property}
\newtheorem{corollary}[property]{Corollary}

\input{abbr.tex}

\def\isubscript{\mathrm{i}}
\def\fsubscript{\mathrm{f}}
\def\traj{\mathcal{I}}
\def\csubscript{\mathrm{c}}
\def\zerovec{\bm{0}}
\newcommand{\zeromat}[2]{\bm{0}_{{#1}, {#2}}}

\DeclareMathOperator*{\minimize}{\mbox{minimize}}
\def\MATbx{\color{blue} \times}
\def\MATbz{\color{blue} 0}
\def\MATcx{\color{cyan} \times}

\def\MATgx{\color{green}\times}
\def\MATgz{\color{green}0}
\def\MATrxb{\color{red}\times}
\def\MATrx{\color{red}  \times}
\def\MATrzb{{\color{red}0}\vspace{-1pt}}
\def\MATrz{\color{red}  0}
\def\MATxxb{\color{gray}\times}
\def\MATxx{\color{gray} \times}
\def\MATxz{\color{gray} 0}
\def\MATzb{0\vspace{-1pt}}
\def\MATz{0\vspace{-2pt}}
\def\alphamax{\alpha_\textnormal{max}}
\def\alphamin{\alpha_\textnormal{min}}

\def\bfcdi{\bfcd_\isubscript}
\def\bfcf{\bfc_\fsubscript}
\def\bfci{\bfc_\isubscript}
\def\bfnabla{\bm{\nabla}}
\def\bfni{\bfn_\isubscript}
\def\bfphi{\bm{\varphi}}
\def\bfrf{\bfr_\fsubscript}
\def\bfri{\bfr_\isubscript}
\def\bfxf{\bfx_\fsubscript}
\def\bfxiinit{\bfxi_\isubscript}
\def\bfxinit{\bfx_\isubscript}

\def\calCf{{\calC_\fsubscript}}
\def\calCi{{\calC_\isubscript}}

\def\lambdaf{\lambda_\fsubscript}

\def\lambdamax{\lambda_\textnormal{max}} 
\def\lambdamin{\lambda_\textnormal{min}} 
\def\omegaimax{\omega_{\isubscript,\textnormal{max}}}
\def\omegaimin{\omega_{\isubscript,\textnormal{min}}}
\def\omegai{\omega_\isubscript}

\def\sc{s_\textnormal{c}}
\def\tc{t_\textnormal{c}}
\def\tswing{t_\textnormal{swing}}

\def\hi{h_\isubscript}
\def\hdi{\dot{h}_\isubscript}
\def\halpha{h_\alpha}
\def\hf{h_\fsubscript}

\newcommand{\BIN}{\begin{bmatrix}}
\newcommand{\sBIN}{\left[\begin{smallmatrix}}
\newcommand{\BOUT}{\end{bmatrix}}
\newcommand{\sBOUT}{\end{smallmatrix}\right]}
\newcommand{\subjto}{\mbox{subject to}}

\begin{document}

\title{Capturability-based Pattern Generation for \\ Walking with Variable Height}

\author{St\'{e}phane~Caron, Adrien~Escande, Leonardo~Lanari, and~Bastien~Mallein
\thanks{%
    St\'{e}phane Caron is with the Laboratoire d'Informatique, de Robotique et de Micro\'{e}lectronique de Montpellier (LIRMM), CNRS--University of Montpellier, Montpellier, France.

    Adrien Escande is with the CNRS-AIST Joint Robotics Laboratory (JRL), UMI3218/RL, Japan. 

    Leonardo Lanari is with the Dipartimento di Ingegneria Informatica, Automatica e Gestionale, Sapienza Universit\`{a} di Roma, Rome, Italy (e-mail: lanari@diag.uniroma1.it).

    Bastien Mallein is with the Laboratoire Analyse, G\'{e}om\'{e}trie et Applications (LAGA), CNRS--Paris 13 University, Villetaneuse, France.

    Corresponding author:~\texttt{stephane.caron@lirmm.fr}.}}

\maketitle

\begin{abstract}
    Capturability analysis of the linear inverted pendulum (LIP) model enabled walking with constrained height based on the \emph{capture point}. We generalize this analysis to the variable-height inverted pendulum (VHIP) and show how it enables 3D walking over uneven terrains based on \emph{capture inputs}. Thanks to a tailored optimization scheme, we can compute these inputs fast enough for real-time model predictive control. We implement this approach as open-source software and demonstrate it in dynamic simulations.
\end{abstract}


\IEEEpeerreviewmaketitle

\section{Introduction}
\label{sec:introduction}

Capturability quantifies the ability of a system to come to a stop. For a humanoid walking in the linear inverted pendulum~(LIP) mode, it is embodied by the \emph{capture point}, the point on the ground where the robot should step in order to bring itself to a stop~\cite{pratt2006humanoids}. In recent years, one of the main lines of research in LIP-based studies has explored the question of walking by feedforward planning and feedback control of the capture point~\cite{koolen2012ijrr, sugihara2009icra, takenaka2009iros, morisawa2012humanoids, englsberger2015tro, griffin2017iros}.

The LIP owes its tractability to two assumptions: no angular-momentum variation around the center of mass (CoM), and a holonomic constraint on the CoM height. As a consequence of the latter, a majority of LIP-based walking controllers assume a flat terrain. Removing this holonomic constraint from the LIP leads to the \emph{variable-height inverted pendulum}~(VHIP) model, for which our understanding is at an earlier stage. Previous studies~\cite{pratt2007icra, ramos2015humanoids, koolen2016humanoids} focused on its balance control for planar motions (sagittal and vertical only). In a preliminary version of this work~\cite{caron2018icra}, we extended the analysis from 2D to 3D balance control. In the present work, we bridge the gap from balancing to walking.

Our contribution is three-fold. First, we provide a necessary and sufficient condition for the capturability of the VHIP model (Section~\ref{sec:analysis}). Second, we show how to turn this condition into an optimization problem (Section~\ref{sec:zero-step}) for which we develop a tailored optimization scheme (Section~\ref{sec:optim}). Finally, we adapt this optimization into a model-predictive walking pattern generator for rough terrains (Section~\ref{sec:one-step}) that we demonstrate in dynamic simulations (Section~\ref{sec:simus}).

\section{Capturability of inverted pendulum models}
\label{sec:analysis}

The critical part of the dynamics for a walking biped lies in the Newton-Euler equation that drives its unactuated floating-base coordinates:
\begin{equation}
    \label{eq:newton-euler}
    \begin{bmatrix} \bfcdd \\ \dot{\bfL}_{\bfc} \end{bmatrix}
    =
    \begin{bmatrix} \frac{1}{m} \bff \\ \bftau_{\bfc} \end{bmatrix}
    +
    \begin{bmatrix} \bfg \\ \bm{0} \end{bmatrix}
\end{equation}
where $\bfc$ is the position of the center of mass (CoM) of the robot, $\bfg$ is the gravity vector (also written $\bfg = -g\bfe_z$ with $g$ the gravitational constant), $m$ is the total robot mass and $\bfL_{\bfc}$ is the angular momentum around $\bfc$. The net contact wrench $(\bff, \bftau_{\bfc})$ consists of the resultant $\bff$ of external contact forces and their moment $\bftau_{\bfc}$ around the CoM.

\subsection{Inverted pendulum models}

The linear inverted pendulum~(LIP)~\cite{kajita2001iros} model is based on two constraints: a constant angular momentum around the center of mass, and a constant CoM height with respect to a reference plane:
\begin{align}
    \label{eq:lipm-assumption-1} \dot{\bfL}_{\bfc} & = \bm{0} \\
    \label{eq:lipm-assumption-2} \bfn \cdot (\bfc - \bfo) & = h
\end{align}
where $\bfn$, $\bfo$ and $h$ are respectively the normal vector, reference point and reference height that define the CoM motion plane. As a consequence of these two assumptions, the Newton-Euler equation~\eqref{eq:newton-euler} simplifies\footnote{
    See \emph{e.g.}~\cite{caron2016tro} for a reminder of the steps of this derivation.
} to:
\begin{equation}
    \label{eq:lipm}
    \bfcdd = \omega^2 (\bfc - \bfr) + \bfg
\end{equation}
where $\omega \defeq \sqrt{g / h}$ is a constant and $\bfr$ is the {center of pressure}~(CoP), or zero-tilting moment point~(ZMP) when there are multiple contacts. A strong limitation of the LIP is the holonomic constraint~\eqref{eq:lipm-assumption-2} on the CoM, which can be kinematically problematic in scenarios such as stair climbing. One line of research sought to overcome this by constraining the CoM to parametric or piecewise-planar surfaces~\cite{morisawa2005icra, zhao2012humanoids}, but the next question of deciding such surfaces based on terrain topology has not attracted a lot of attention so far.

The variable-height inverted pendulum~(VHIP) model strips away the holonomic CoM constraint altogether~\cite{hopkins2014humanoids, kamioka2015iros}. Its equation of motion is:
\begin{equation}
    \label{eq:vhip}
    \bfcdd = \lambda (\bfc - \bfr) + \bfg
\end{equation}
where $\lambda > 0$ is now a time-varying stiffness\footnote{
    All quantities being normalized by mass, we call $\lambda$ a stiffness although its unit is $\textnormal{s}^{-2}$ and not $\textnormal{kg}.\textnormal{s}^{-2}$. Similarly, we will refer to frequencies (unit:~$\textnormal{s}^{-1}$) as dampings.
} coefficient. The two control inputs of this system are the center of pressure $\bfr$ and the stiffness $\lambda$, as shown in Figure~\ref{fig:pendulums}.

\subsection{Feasibility conditions}

To be \emph{feasible}, the CoP $\bfr$ must belong to the contact area $\calC$ under the supporting foot. This area is also time-varying but changes only discretely. The transition from one support contact to the next is called a \emph{contact switch}. A trajectory with $N$ contact switches is called an $N$-step trajectory. 

We assume that all contact areas are planar and polygonal. Let us denote by $\bfo$ the center of the area $\calC$ and by $\bfn$ its normal (such that $\bfn \cdot \bfe_z \neq 0$). The CoP $\bfr$ belongs to the plane of contact if and only if $(\bfr - \bfo) \cdot \bfn = 0$. The \emph{height} of the CoM $\bfc$ above the contact area is the algebraic distance $h(\bfc)$ such that $\bfc - h(\bfc) \bfe_z$ belongs to the contact plane:
\begin{equation}
    h(\bfc) \defeq \frac{(\bfc - \bfo) \cdot \bfn}{(\bfe_z \cdot \bfn)}
\end{equation}
Note how, when walking on a horizontal floor, $\bfe_z$ and $\bfn$ are aligned and $h$ is simply the $z$ coordinate of the center of mass. 

To be \emph{feasible}, the stiffness $\lambda$ must be non-negative by unilaterality of contact. We furthermore impose that $\lambda \in [\lambdamin, \lambdamax]$ so that contact pressure is never exactly zero and remains bounded. Note that we do not model Coulomb friction conditions here: having found in a previous work that CoP feasibility constraints are usually more stringent than friction constraints when walking over uneven terrains~\cite{caron2017iros}, we assume sufficient friction in the present study.

\begin{figure}[t]
    \centering
    \includegraphics[width=\columnwidth]{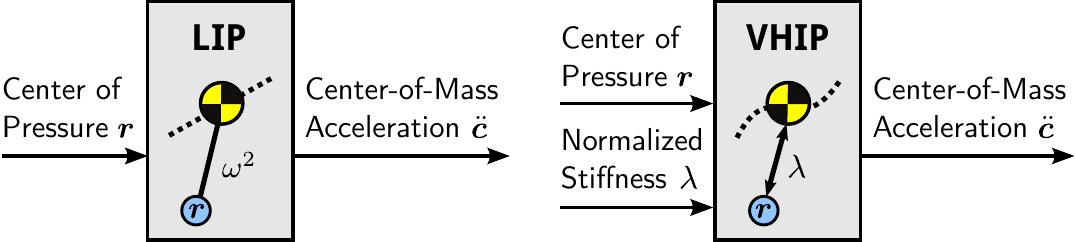}
    \caption{
        Inputs and output of the \emph{linear inverted pendulum}~(LIP) and \emph{variable-height inverted pendulum}~(VHIP) models. In the former, the stiffness coefficient is fixed to a constant $\omega^2$, while it becomes an additional input $\lambda$ of the latter.
    }
    \label{fig:pendulums}
\end{figure}

An input function $t \mapsto (\lambda(t), \bfr(t))$ is \emph{feasible} when both $\lambda(t)$ and $\bfr(t)$ are feasible at all times $t$. A general control problem is to find a feasible input function such that the resulting output trajectory $\bfc(t)$ has certain properties. For the locomotion problem of ``getting somewhere'', we will focus on the property of converging to a desired location.

\subsection{Capture inputs and capture trajectories}

A natural choice of the pendulum state consists of its CoM position and velocity~$\bfx = (\bfc, \bfcd)$. 
\begin{definition}[Static equilibrium]
    A state $\bfx = (\bfc, \bfcd)$ is a static equilibrium when its velocity $\bfcd$ is zero and can remain zero with suitable constant controls $(\lambdaf, \bfrf)$.
\end{definition}
Static equilibria, also called \emph{capture states}~\cite{pratt2006humanoids}, are the targets of capturability analysis, the desired locations that the CoM should converge to. A static equilibrium is characterized by its CoM position $\bfcf$ and the contact $\calCf$ upon which it rests at height $\hf = h(\bfcf)$. The only control input $\lambdaf, \bfrf$ that maintains the pendulum in static equilibrium is such that $\bfcf = \bfrf - \bfg / \lambdaf$, that is:
\begin{align}
    \label{eq:lambdaf}
    \lambdaf(\bfcf) & = \frac{g}{\hf} &
    \bfrf(\bfcf) & = \bfcf - \hf \bfe_z
\end{align}

Given an $N$-step contact sequence, we say that a state $\bfxinit$ is ($N$-step) \emph{capturable} when there exists a feasible input function $\lambda(t), \bfr(t)$ such that applying Equation~\eqref{eq:vhip} from $\bfxinit$ brings the system asymptotically to an equilibrium $\bfxf$. We call such functions \emph{capture inputs} of the capturable state $\bfxinit$, and denote their set by $\traj_{\bfxinit, \bfxf}$. We call \emph{capture trajectory} the CoM trajectory $\bfc(t)$ resulting from a capture input. In what follows, we use the subscript $\square_\isubscript$ to denote the ``initial'' or instantaneous state of the system, and the subscript $\square_\fsubscript$ for its ``final'' or asymptotic state.

By definition, $\bfxinit$ is capturable if and only if there exists $\bfxf$ such that $\traj_{\bfxinit,\bfxf} \neq \emptyset$. The set $\traj_{\bfxinit,\bfxf}$ contains however several solutions, including mathematical oddities such as functions with ever-increasing frequencies. In what follows, we restrict it to inputs that converge asymptotically:
\begin{equation}
  \traj_{\bfxinit,\bfxf}^c = \left\{ (\lambda(t),\bfr(t)) \in \traj_{\bfxinit,\bfxf} : \begin{array}{l} \lim_{t \to \infty} \lambda(t) = \lambdaf\\ \lim_{t \to \infty} \bfr(t) = \bfrf\end{array} \right\}
\end{equation}
Property~\ref{prop:capture-traj} shows that this choice does not cause any loss of generality (see Appendix~\ref{app:math-background}).

\subsection{Dichotomy of the components of motion}

One of the main findings in the control of the LIP model is to focus on its divergent component of motion~(DCM), the capture point. Among other benefits, controlling only the capture point reduces second-order dynamics to first order and maximizes the basin of attraction of feedback controllers~\cite{koolen2012ijrr, sugihara2009icra, takenaka2009iros}. An important step in the analysis of the VHIP model is therefore to identify its DCM. Interestingly, the answer can be found in a study of motorcycle balance~\cite{hauser2004cdc}, which we recall here.

The equation of motion~\eqref{eq:vhip} of the VHIP can be interpreted as either nonlinear or linear time-variant, depending on whether one focuses respectively on feedback or open-loop control. We focus on the latter for walking pattern generation. Let us then rewrite this equation as a first-order linear time-variant system:
\begin{equation}
    \BIN \bfcd \\ \bfcdd \BOUT
    =
    \BIN \zerovec & \bfI \\ \lambda \bfI & \zerovec \BOUT 
    \BIN \bfc \\ \bfcd \BOUT
    +
    \BIN \zerovec \\ \bfg - \lambda \bfr \BOUT
\end{equation}
where $\bfI$ is the $3 \times 3$ identity matrix. This equation has the form $\bfxd = \bfA(t) \bfx + \bfb(t)$ where the system matrix $\bfA$ depends on the stiffness input $\lambda$, while the forcing term $\bfb$ varies with both inputs $\lambda$ and $\bfr$.

Hauser \emph{et al.} showed~\cite{hauser2004cdc} how to obtain an \emph{exponential dichotomy}~\cite{coppel1971sdeds} of the state $\bfx$ (that is, how to decompose it into convergent and divergent components) by applying a change of coordinates $\bfx = \bfS \bfz$ with:
\begin{equation}
    \bfS 
    = \frac{1}{\gamma + \omega} 
    \BIN \bfI & \bfI \\ -\omega \bfI & \gamma \bfI \BOUT
    \Longleftrightarrow
    \bfS^{-1}
    = 
    \BIN \gamma \bfI & -\bfI \\ \omega \bfI & +\bfI \BOUT
    \label{eq:S-Sinv}
\end{equation}
The two functions $\gamma(t)$ and $\omega(t)$ are \emph{positive} and of class $\calC^1$. We will refer to them as \emph{dampings} in accordance with their (normalized) physical unit. The new state vector $\bfz$ then consists of two components $\bfzeta$ and $\bfxi$ defined by:
\begin{align}
    \bfzeta & = \gamma \bfc - \bfcd \\
    \bfxi & = \omega \bfc + \bfcd \label{eq:def-xi}
\end{align}
They respectively correspond to the \emph{convergent} and \emph{divergent component of motion}~(DCM). In the case of a linear inverted pendulum with constant $\omega$, the DCM $\bfxi$ is simply proportional to the capture point $\bfc + \bfcd / \omega$.

Collectively, the state vector $\bfz$ is subject to:
\begin{align}
\bfzd & = \tilde{\bfA} \bfz + \tilde{\bfb} \\
    \tilde{\bfA} & = \bfS^{-1} (\bfA \bfS - \dot{\bfS}) \\
    \tilde{\bfb} & = \bfS^{-1} \bfb
\end{align}
The calculation of $\tilde{\bfb}$ is straightforward. That of $\tilde{\bfA}$ yields:
\begin{align}
    \tilde{\bfA} & = \frac{1}{\gamma + \omega}
    \BIN 
        (\dot{\gamma} - \gamma \omega - \lambda) \bfI &
        (\dot{\gamma} + \gamma^2 - \lambda) \bfI \\
        (\dot{\omega} - \omega^2 + \lambda) \bfI &
        (\dot{\omega} + \omega \gamma + \lambda) \bfI
    \BOUT
\end{align}
To decouple the system, we can eliminate non-diagonal terms in this state matrix by imposing the two following Riccati equations:
\begin{align}
    \dot{\gamma} & = \lambda - \gamma^2 \label{eq:riccati-gamma} \\
    \dot{\omega} & = \omega^2 - \lambda \label{eq:riccati-omega}
\end{align}
This results in the following state dynamics:
\begin{equation}
    \label{eq:decoupledsystem}
    \bfzd 
    =
    \BIN \bfzetad \\ \bfxid \BOUT
    =
    \BIN -\gamma \bfI & \zerovec \\ \zerovec & \omega \bfI \BOUT
    \BIN \bfzeta \\ \bfxi \BOUT
    +
    \BIN \lambda \bfr - \bfg \\ \bfg - \lambda \bfr \BOUT
\end{equation}

The linear time-varying system has thus been decoupled into two linearly independent components $\bfzeta$ and $\bfxi$ that evolve according to their own dynamics, provided that there exists two $\mathcal{C}^1$ positive finite solutions to~\eqref{eq:riccati-gamma} and \eqref{eq:riccati-omega}. A proof of this and a detailed analysis of damping functions are given in Appendix~\ref{app:math-riccati-solutions}.

Although we defer the detailed analysis of damping functions to this Appendix, its takeaway point is that $\gamma$ and $\omega$ are, in themselves, convergent and divergent components. A parallel can be drawn between the DCM--CoP and $\omega$--$\lambda$ systems:
\begin{itemize}
    \item In the LIP, with the CoP restricted to a support area, the CoP is a repulsor of the DCM, and the DCM is controllable if and only if it is above the support area~\cite{sugihara2009icra}.
    \item In the VHIP, with $\lambda$ restricted to $[\lambdamin, \lambdamax]$, $\lambda$ is a repulsor of $\omega$, and $\omega$ is controllable if and only if it belongs to $[\sqrt{\lambdamin}, \sqrt{\lambdamax}]$ (Property~\ref{prop:omega-bounds}).
\end{itemize}
This remark is central to the walking pattern generation method in Section~\ref{sec:zero-step}, which reduces three-dimensional capturability to one dimension.

It can be shown that, regardless of the initial state $\bfxinit$ of the system, \emph{any} input function from $\traj_{\bfxinit,\bfxf}^c$ makes $\bfzeta$ converge as well (Property~\ref{prop:zeta-conv}), owing it its name of {convergent} component of motion. From a control perspective, spending additional inputs to control this component is not necessary and can even be wasteful.\footnote{
    For the LIP, it reduces the basin of attraction of feedback controllers~\cite{sugihara2009icra}.
} The main concern of capturability analysis is therefore to prevent the other component $\bfxi$ from diverging.

\subsection{Boundedness condition}

The divergent component of motion $\bfxi$ corresponding to the damping $\omega$ is subject to the differential equation:
\begin{equation}
    \label{eq:xi-def}
    \bfxid = \omega \bfxi + \bfg - \lambda \bfr
\end{equation}
The general solution to this equation is given by:
\begin{equation}
    \label{eq:xi(t)}
    \bfxi(t) = \left(\bfxi(0) + \int_0^t e^{-\Omega(\tau)} (\bfg - \lambda (\tau) \bfr(\tau)) \dd{\tau}\right) e^{\Omega(t)}
\end{equation}
where $\Omega(t) = \int_0^t \omega(t) \dd{t}$. In our working assumptions, this integral is well-defined and finite (details in Appendix~\ref{app:dcm}). 

Set aside the particular condition that we are about to discuss, the function $\bfxi(t)$ diverges as $t \to \infty$, giving $\bfxi$ its name of {divergent} component of motion (DCM)~\cite{takenaka2009iros}.\footnote{
    More specifically, our analysis considers a \emph{time-varying} divergent component of motion~\cite{hopkins2014humanoids}. While previous works such as~\cite{takenaka2009iros, englsberger2015tro, hopkins2014humanoids} chose to write their DCMs as positions $\bfc + \bfcd / \omega$, we cast them as velocities here to simplify calculations (consider the derivative of a product $uv$ compared to that of a ratio $u/v$). The formula of the DCM itself is not a crucial design choice, as we will discuss at the end of this Section.
} However, a careful match between future capture inputs and the initial condition $\bfxiinit$ can guarantee that $\bfxi(t)$ converges as well. This choice is known as the \emph{boundedness condition}~\cite{lanari2014humanoids}:

\begin{property}[Boundedness condition]
    \label{prop:xi-conv}
    Consider an input function $\lambda(t), \bfr(t)$ such that $\lim_{t \to \infty} \lambda(t) = \lambdaf$ and $\lim_{t\to \infty} \bfr(t) = \bfrf$. Then, there exists a unique $\bfxiinit = \omegai \bfci + \bfcdi$ such that the solution $\bfxi$ of \eqref{eq:xi-def} with $\bfxi(0) = \bfxiinit$ remains finite at all times. This initial condition is given by:
    \begin{equation}
        \label{eq:xi-initial}
        \boxed{\bfxiinit = \int_0^\infty e^{-\Omega(t)} (\lambda(t) \bfr(t) - \bfg) \dd{t}}
    \end{equation}
    where $\omegai$ is the initial value of the unique bounded solution $\omega$ to the Riccati equation $\omegad = \omega^2 - \lambda$ and $\Omega$ is the antiderivative of $\omega$ such that $\Omega(0)=0$.
\end{property}

The proof of this property is given in Appendix~\ref{app:dcm}. In the familiar setting of the LIP where $\lambda = \omega_\csubscript^2$ is a constant, taking a constant CoP $\bfr_\csubscript$ yields:
\begin{align}
    \omega_\csubscript \bfci + \bfcdi =
    \int_0^\infty (\omega_\csubscript^2 \bfr_\csubscript - \bfg) e^{-\omega_\csubscript t} \dd{t} = 
    \omega_\csubscript \bfr_\csubscript - \frac{\bfg}{\omega_\csubscript} 
\end{align}
Over horizontal coordinates, this equation implies that $\bfr_\csubscript^{xy} = \bfci^{xy} + {\bfcdi^{xy}}/{\omega_\csubscript}$, \emph{i.e.} the CoP is located at the capture point. Over the $z$ coordinate, it yields $\omega_\csubscript = \sqrt{g / h}$, the well-known expression of the natural frequency of the LIP. Overall, the boundedness condition characterizes the capturability of the LIP. We will now conclude our capturability analysis by showing how this is also the case for the VHIP.

\subsection{Capturability of the variable-height inverted pendulum}

What we have established so far is a necessary condition: if an input function belongs to $\traj_{\bfxinit,\bfxf}^c$, then it is feasible, converging and satisfies the boundedness condition. The key result of our capturability analysis is that this condition is also sufficient, \emph{i.e.} it \emph{characterizes} capture inputs:

\begin{property}
    \label{prop:capture-inputs}
    Let $\bfxinit = (\bfci, \bfcdi)$ denote a capturable state and $\bfxf = (\bfcf, \zerovec)$ a static equilibrium. Then, $t \mapsto \lambda(t), \bfr(t)$ is a capture input from $\bfxinit$ to $\bfxf$ if and only if:
    \begin{enumerate}[(i)]
		\item its values $\lambda(t)$ and $\bfr(t)$ are feasible for all $t \geq 0$,
        \item $\lim_{t \to \infty} \lambda(t) = \lambdaf(\bfcf)$ and $\lim_{t \to \infty} \bfr(t) = \bfrf(\bfcf)$,
        \item it satisfies the boundedness condition:
            \begin{equation}
                \label{eq:boundedness}
                \int_0^\infty (\lambda(t) \bfr(t) - \bfg) e^{-\Omega(t)} \dd{t} = \omegai \bfci + \bfcdi
            \end{equation}
    \end{enumerate}
    where $\omegai$ is the initial value of the unique bounded solution $\omega$ to the Riccati equation $\omegad = \omega^2 - \lambda$ and $\Omega$ is the antiderivative of $\omega$ such that $\Omega(0)=0$.
\end{property}

A proof of this property is given in Appendix~\ref{app:capture-inputs}. In the remainder of this manuscript, we will see how this is not only a theoretical but also a practical result with applications to balance control and walking pattern generation.

A noteworthy methodological point here is that the expression of the divergent component of motion is not unique. Rather, \emph{a} DCM is chosen by the roboticist. For example, in~\cite{caron2018icra} we considered a different DCM $\widetilde{\bfxi} \defeq \omega (\bfc - \bfr) + \bfcd - \bfrd$ yielding a boundedness condition written:
\begin{equation}
  \label{eq:boundednessOld}
  \int_0^\infty (\ddot{\bfr}(t)-\bfg) e^{-\Omega(t)} \dd{t} = \omegai \bfci + \bfcdi
\end{equation}
This condition is the same as~\eqref{eq:boundedness}, which can be seen by applying a double integration by parts. In the present work, we chose the DCM from Equation~\eqref{eq:xi-def} as it makes calculations simpler. We preferred a velocity-based rather than position-based DCM for the same reason, as a the differential of a product involves less operations than that of a ratio.

\section{Balance control with variable height}
\label{sec:zero-step}

Let us consider first the problem of balance control, \emph{i.e.} {zero-step} capturability. The robot pushes against a stationary contact area $\calC$ in order to absorb the linear momentum of its initial state $\bfxinit$, eventually reaching a static equilibrium $\bfxf$. This level of capturability enables push recovery~\cite{pratt2006humanoids, stephens2007humanoids, koolen2016humanoids, yamamoto2016ras} up to post-impact fall recovery in worst-case scenarios~\cite{samy2017humanoids, delprete2017hal}.

The gist of the method we propose thereafter is to reduce the three-dimensional capturability condition over $\lambda,\bfr$ (Property~2) into a one-dimensional condition over $\lambda$. To do so, we couple the evolution of $\lambda$ and $\bfr$ by a suitably-defined intermediate variable $s$. The complete pipeline goes as follows:
\begin{itemize}
    \item Change variable from time $t$ to a new variable $s$
    \item Define the CoP evolution $\bfr(s)$ as a function of $s$
    \item Reduce capturability to an optimization over $\omega(s)$
    \item Compute the optimal solution $\omega^*(s)$ of this problem
    \item Change variable from $s$ to time $t$
\end{itemize}
From $\omega^*(t)$, it is then straightforward to compute the full capture input $t \mapsto \lambda^*(t), \bfr^*(t)$ as well as the capture trajectory $\bfc^*(t)$. Let us now detail each step of this pipeline.

\subsection{Change of variable}
\label{sec:timeless}

Define the adimensional quantity:
\begin{align}
    s(t) = e^{-\Omega(t)}
\end{align}
This new variable ranges from $s=1$ when $t=0$ to $s \to 0$ when $t \to \infty$. Its time derivatives are:
\begin{align}
    \sd(t) & = -\omega(t) s(t) \label{eq:squared-vel} &
    \sdd(t) & = \lambda(t) s(t)
\end{align}
Owing to the bijective mapping between $t$ and $s$, we can define $\omega$, $\gamma$ and $\lambda$ as functions of $s$ rather than as functions of $t$. This approach is \emph{e.g.} common in time-optimal control~\cite{pham2018tro}. Let us denote by $\square'$ derivation with respect to $s$, as opposed to $\dot{\square}$ for derivation with respect to $t$. The Riccati equation~\eqref{eq:riccati-omega} of $\omega$ becomes:
\begin{align}
    \label{eq:riccati-lambda-s}
    \lambda & = \omega^2 - \omegad = \omega^2 - \sd \omega' = \omega (\omega + s \omega') = \omega (s \omega)'
\end{align}
Injecting this expression into the time integral~\eqref{eq:boundedness} of the boundedness condition yields:
\begin{equation}
    \int_{0}^\infty (\lambda(t) \bfr(t) - \bfg) s(t) \dd{t} 
    = \int_{0}^1 (\omega (s \omega)' \bfr(s) - \bfg) \frac{\dd{s}}{\omega}
\end{equation}
We can then characterize capture inputs as functions of $s$:

\begin{property}
    \label{prop:capture-inputs-s}
    Let $\bfxinit = (\bfci, \bfcdi)$ denote a capturable state and $\bfxf = (\bfcf, \bm{0})$ a static equilibrium. Then, $s \mapsto \lambda(s), \bfr(s)$ is a capture input from $\bfxinit$ to $\bfxf$ if and only if:
    \begin{enumerate}[(i)]
		\item its values $\lambda(s)$ and $\bfr(s)$ are feasible for all $s \in [0, 1]$,
        \item $\lim_{s \to 0}\lambda(s) = \lambdaf(\bfcf)$ and $\lim_{s \to 0}\bfr(s) = \bfrf(\bfcf)$,
        \item it satisfies the boundedness condition:
            \begin{equation}
                \label{eq:boundedness-s}
                \int_0^1 \bfr(s) (s \omega)' \dd{s} - \bfg \int_0^1 \frac{\dd{s}}{\omega(s)} = \omegai \bfci + \bfcdi
            \end{equation}
    \end{enumerate}
    where $\omegai$ denotes the initial value (at $s=1$) of the solution $\omega$ to the differential equation $\omega (s \omega)' = \lambda$.
\end{property}

With this reformulation, the infinite-time integral has become finite over the $[0, 1]$ interval and the antiderivative $\Omega$ has been replaced by $\omega$ itself.

\subsection{Time-varying CoP strategy}
\label{sec:tv-cop}

The CoP and gravity terms of the boundedness condition~\eqref{eq:boundedness-s} can be separated by projecting them in the (non-orthogonal) basis $(\bfe_x, \bfe_y, \bfn)$:
\begin{align}
    \label{eq:boundedness-rxy}
    \int_0^1 \bfr^{xy}(s) (s \omega)' \dd{s} = \omegai \bfci^{xy} + \bfcdi^{xy}
    \\
    \int_0^1 \frac{\dd{s}}{\omega(s)} = \frac{\omegai \hi + \hdi}{g} 
    \label{eq:boundedness-omega}
\end{align}
where $\hi \defeq h(\bfci)$ is the initial CoM height and $\hdi$ its velocity, both known from the initial state $\bfci$. The gravity term~\eqref{eq:boundedness-omega} only involves $\omega(s)$, but the CoP term~\eqref{eq:boundedness-rxy} involves both $\bfr(s)$ and $\omega(s)$. We reduce it to an integral over $\omega(s)$ by making the CoP move along a line segment\footnote{
    Strategies with two control points such as this one are the simplest one can imagine, in the sense that it is in general impossible to realize three-dimensional balance control with a stationary CoP (see Appendix~\ref{app:fixed-cop}).
} from $\bfri$ to $\bfrf$:
\begin{align}
    \label{eq:cop-line} \bfr(s) &= \bfrf + (\bfri - \bfrf) f(s \omega)
\end{align}
where $f$ can be any smooth function that satisfies:
\begin{itemize}
\item $f(\omegai) = 1$: the CoP is initially located at $\bfri$,
\item $f(0) = 0$: the CoP converges to $\bfrf$,
\item \emph{$f$ is increasing}: we exclude solutions where the CoP would move back and forth along the line segment,
\item \emph{$f$ is integrable}: let $F$ denote its antiderivative such that $F(0)=0$. It is positive by monotonicity of $f$.
\end{itemize}
The final CoP location $\bfrf$ is already known from the desired capture state $\bfxf$, but the instantaneous CoP $\bfri$ is a decision variable. In the example depicted in Figure~\ref{fig:linear-cop}, it will be chosen on the other side of the line $\bfci + \mathbb{R} \bfcdi$ compared to $\bfrf$ in order to progressively reorient $\bfcd(t)$ toward the capture state $\bfrf$, similarly to the behavior observed in the LIP with linear capture-point feedback control~\cite{sugihara2009icra, morisawa2012humanoids, englsberger2015tro}.

\begin{figure}[t]
    \centering
    \includegraphics[height=2.8cm]{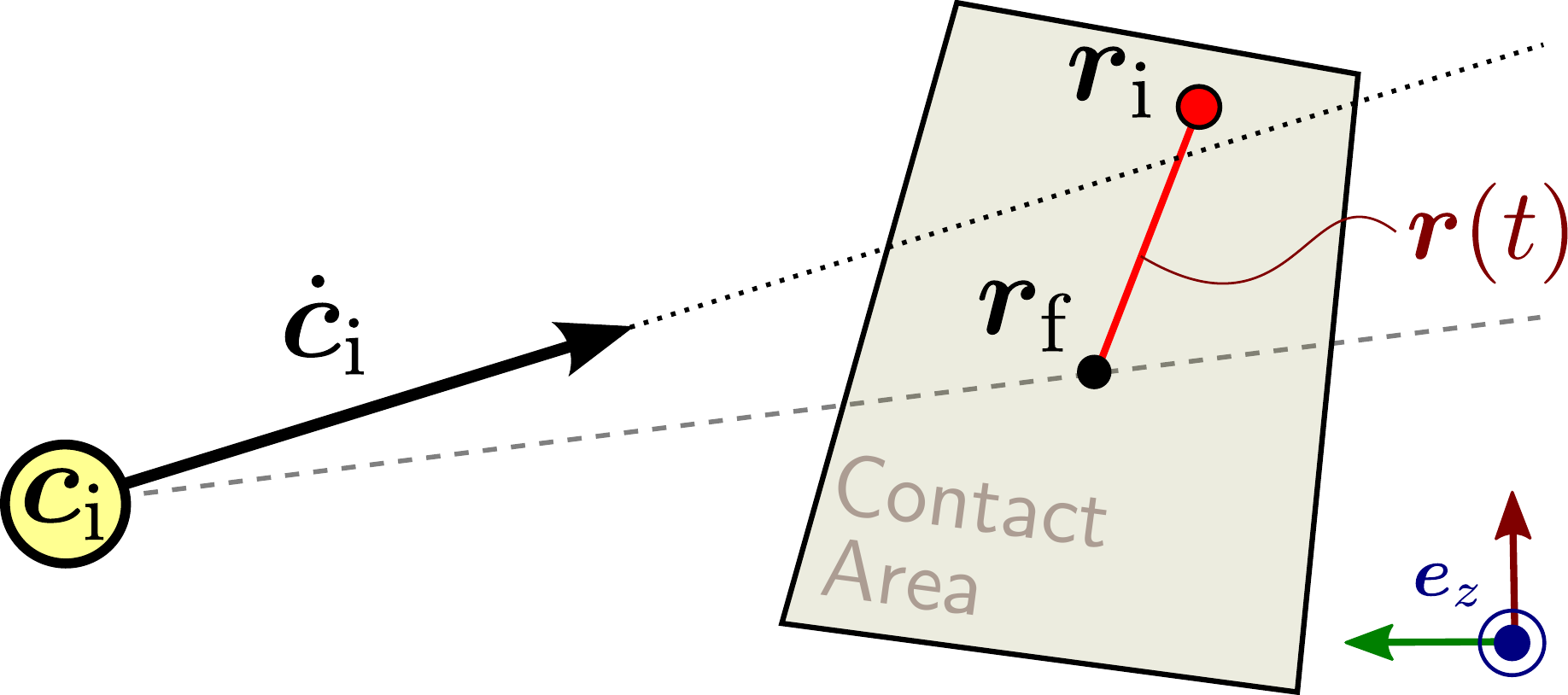}
    \caption{
        \textbf{Linear time-varying CoP trajectory.}
        Variations of the center of pressure inside the contact area allow the robot to reorient its velocity towards a capture state without re-stepping. 
    }
    \label{fig:linear-cop}
\end{figure}

Combining Equations~\eqref{eq:boundedness-rxy} and \eqref{eq:cop-line} yields:
\begin{equation}
    \label{eq:tv-cop-ri}
    \bfri^{xy} = \bfrf^{xy} + \frac{\omegai (\bfci^{xy} - \bfrf^{xy}) + \bfcdi^{xy}}{F(\omegai)} 
\end{equation}
At this stage, the roboticist can explore different CoP strategies via the choice of a function $F$. We choose a power law parameterized by $\alpha \in (0, 1)$:
\begin{equation}
    f(s \omega) = \left(\frac{s \omega}{\omegai}\right)^\frac{\alpha}{1 - \alpha} \Longrightarrow F(\omegai) = (1 - \alpha) \omegai
\end{equation}
With this choice, the horizontal coordinates of the instantaneous CoP imposed by the boundedness condition become:
\begin{equation}
    \label{eq:alpha-rixy}
    \bfri^{xy} = \bfrf^{xy} + \frac{1}{1 - \alpha} \left[\bfci^{xy} + \frac{\bfcdi^{xy}}{\omegai} - \bfrf^{xy}\right]
\end{equation}
where we recognize the same expression as in capture-point feedback control of the LIP~\cite{sugihara2009icra, morisawa2012humanoids, englsberger2015tro}. Note that the three-dimensional position $\bfri$ of the CoP is readily available from $\bfri^{xy}$ by vertical projection:
\begin{equation}
    \label{eq:below-alpha-rixy}
    \bfri = \bfri^{xy} - h(\bfri^{xy}) \bfe_z
\end{equation}

The current state $\bfxinit$ and target capture state $\bfxf$ being given, the only decision variable left on the right-hand side of Equation~\eqref{eq:alpha-rixy} is $\omegai$. At this point, we have almost reduced the CoP capturability conditions to $\omega$: with respect to Property~\ref{prop:capture-inputs-s}, $\bfr(s)$ converges to $\bfrf$ (\emph{ii}) and satisfies the boundedness condition by selecting $\bfri$ from Equation~\eqref{eq:alpha-rixy} (\emph{iii}). We now need to make sure that the CoP trajectory is feasible (\emph{i}).

By convexity of the contact area, the CoP trajectory is feasible if and only if both its ends $\bfri$ and $\bfrf$ are in the area. We assume the latter does by construction. For the former, the constraint that $\bfri$ belongs to the contact polygon can be described in halfspace representation by a matrix-vector inequality $\bfH \bfri^{xy} \leq \bfp$, with $\bfH$ an $m \times 2$ matrix and $\bfp$ an $m$-dimensional vector. For example, a rectangular contact area written in the contact frame $(\bft, \bfb, \bfn)$ as:
\begin{align}
    \pm \bft \cdot (\bfri - \bfo) & \leq X \\
    \pm \bfb \cdot (\bfri - \bfo) & \leq Y 
\end{align}
can be reformulated equivalently in the horizontal plane:
\begin{align}
    \pm (\bfb \times \bfe_z) (\bfri^{xy} - \bfo^{xy}) & \leq X (\bfe_z \cdot \bfn) \\
    \pm (\bft \times \bfe_z) (\bfri^{xy} - \bfo^{xy}) & \leq Y (\bfe_z \cdot \bfn)
\end{align}
Injecting Equation~\eqref{eq:alpha-rixy} into inequalities $\bfH \bfri^{xy} \leq \bfp$ yields:
\begin{equation}
    \label{eq:ineq-omegai}
    \left[\alpha \bfH \bfrf^{xy} + (1 - \alpha) \bfp - \bfH \bfci^{xy} \right] \omegai \geq \bfH \bfcdi^{xy}
\end{equation}
Each line of this vector inequality $\bfu \omegai \geq \bfv$ provides a lower or upper bound on $\omegai$ depending on the sign of the factor in front of it:
\begin{align}
    \label{eq:calc-omegaimin}
    \omegaimin & = \max\left(\sqrt{\lambdamin}, \max_j \left\{ \frac{v_j}{u_j}, u_j > 0 \right\}\right) \\
    \label{eq:calc-omegaimax}
    \omegaimax & = \min\left(\sqrt{\lambdamax}, \min_j \left\{ \frac{v_j}{u_j}, u_j < 0 \right\}\right)
\end{align}
We have thus reduced the feasibility condition on the CoP to an inequality constraint on $\omegai$:
\begin{equation}
    \label{eq:zs-last-step}
    \forall s, \bfr(s) \in \calC \Longleftrightarrow \omegaimin \leq \omegai \leq \omegaimax
\end{equation}

Overall, the CoP strategy allows us to reduce the three-dimensional capturability condition over $s \mapsto \lambda(s), \bfr(s)$ (Property~\ref{prop:capture-inputs-s}) into a one-dimensional condition over $s \mapsto \lambda(s)$:

\begin{property}[1D capturability]
    \label{prop:1d-capturability}
    Let $\bfxinit = (\bfci, \bfcdi)$ denote a capturable state and $\bfxf = (\bfcf, \bm{0})$ a static equilibrium. Under the time-varying CoP strategy:
    \begin{align}
        \bfr(s) & = \bfrf + (\bfri - \bfrf) \left(\frac{s \omega}{\omegai}\right)^\frac{\alpha}{1 - \alpha}
        & \alpha \in (0, 1),
    \end{align}
    $s \mapsto \lambda(s)$ is a capture input from $\bfxinit$ to $\bfxf$ if and only if:
    \begin{enumerate}[(i)]
		\item $\omegai \in [\omegaimin, \omegaimax]$ and $\forall s \in [0, 1], \lambda(s) \in [\lambdamin, \lambdamax]$,
        \item $\lim_{s \to 0}\lambda(s) = \lambdaf(\bfcf)$,
        \item it satisfies the boundedness condition:
            \begin{equation}
                \int_0^1 \frac{\dd{s}}{\omega(s)} = \frac{\omegai \hi + \hdi}{g} 
            \end{equation}
    \end{enumerate}
    where $\omegai$ denotes the initial value (at $s=1$) of the solution $\omega$ to the differential equation $\omega (s \omega)' = \lambda$.
\end{property}

\subsection{Formulation as an optimization problem}
\label{sec:formulating}

Let us compute piecewise-constant functions $s \mapsto \lambda(s)$ that satisfy the three conditions from Property~\ref{prop:1d-capturability}. We partition the interval $[0, 1]$ into $n-1$ fixed segments $0 = s_0 < s_1 < \ldots < s_{n-1} < s_n = 1$ such that $\forall s \in (s_j, s_{j+1}], \lambda(s) = \lambda_j$. Note how the interval is closed to the right ($\lambda(s_{j+1}) = \lambda_j$) but open to the left. Define:
\begin{align}
    \label{eq:varphi-def}
    \varphi(s) & \defeq s^2 \omega^2 & \delta_j & \defeq s_{j+1}^2 - s_j^2
\end{align}
The quantity $\varphi$ represents a squared velocity and is commonly considered in time-optimal retiming~\cite{pham2018tro}. Remarking that $\varphi' = 2 s \lambda$ from the Riccati equation~\eqref{eq:riccati-lambda-s}, we can directly compute $\varphi(s)$ for $s \in [s_j, s_{j+1}]$ as:
\begin{equation}
    \label{eq:varphi-s}
    \varphi(s) = \sum_{k=0}^{j-1} \lambda_k \delta_k + \lambda_j (s^2 - s_j^2) = \varphi(s_j) + \lambda_j (s^2 - s_j^2)
\end{equation}
In what follows, we use the shorthand $\varphi_j \defeq \varphi(s_j)$. The values $\lambda(s)$ and $\omega(s)$ for $s \in (s_j, s_{j+1}]$ can be computed back from $\varphi$ using Equations~\eqref{eq:varphi-def} and \eqref{eq:varphi-s}:
\begin{align}
    \label{eq:lambda-omega-from-phi}
    \lambda_j & = \frac{\varphi_{j+1} - \varphi_{j}}{\delta_j} &
    \omega(s) & = \frac{1}{s} \sqrt{\varphi_j + \lambda_j (s^2 - s_j^2)}
\end{align}
Owing to this property, we choose the vector $\bfvarphi \defeq \BIN \varphi_1 \ldots \varphi_n \BOUT$ to be the decision variable of our optimization problem. Note that this vector starts from $\varphi_1$, as $\varphi_0 = 0$ by definition.

\paragraph{Feasibility (i)} noting how $\varphi_n = \omegai^2$ from the equation above, both feasibility conditions can be expressed as:
\begin{align}
    & \omegaimin^2 \leq \varphi_n \leq \omegaimax^2 \\
    & \forall j < n,\ \lambdamin \delta_j \leq \varphi_{j+1} - \varphi_j \leq \lambdamax \delta_j
\end{align}

\paragraph{Convergence (ii)}
$\lambda(s)$ converges to $\lim_{s \to 0} \lambda(s) = \lambda_1 = \varphi_1 / \delta_0$. Convergence to $\lambdaf$ can thus be expressed as:
\begin{equation}
    \varphi_1 = \delta_0 \lambdaf = \frac{\delta_0 g}{\hf}
\end{equation}
(Recall that $\varphi_1$ corresponds to the last time interval by definition of $s$.) The parameter $\hf$ corresponds to the CoM height of the capture state. 

\paragraph{Boundedness (iii)}
the variables $\varphi_j$ can also be used to express the integral as a finite sum:
\begin{align}
    \int_0^1 \frac{\dd{s}}{\omega(s)}
    & = \sum_{j=0}^{n-1} \int_{s_j}^{s_{j+1}} \frac{s \dd{s}}{\sqrt{\varphi_j + \lambda_j (s^2 - s_j^2)}} \\
    & = \sum_{j=0}^{n-1} \int_{0}^{\delta_j} \frac{\dd{v}}{2 \sqrt{\varphi_j + \lambda_j v}} \\
    & = \sum_{j=0}^{n-1} \frac{1}{\lambda_j} \left[\sqrt{\varphi_j + \lambda_j \delta_j} - \sqrt{\varphi_j}\right] \\
    & = \sum_{j=0}^{n-1} \frac{\delta_j}{\sqrt{\varphi_{j+1}} + \sqrt{\varphi_j}} \label{eq:conv-obj}
\end{align}

To complete the optimization problem, we add a regularizing cost function over variations of $\lambda$ so that the ideal behavior becomes a constant $\lambda = \omega_\csubscript^2$, \emph{i.e.} the LIP model. This way, height variations are only added when required. This behavior is \emph{e.g.} depicted in Figure~\ref{fig:zero-step-behavior}, where linear capture trajectories are used until the CoP reaches the boundary of the support area and height variations are used for additional braking.

The complete optimization problem is assembled in Equation~\eqref{eq:optim-full}. We will refer to it as the \emph{capture problem}.

\floatname{algorithm}{Capture Problem}
\renewcommand{\thealgorithm}{}
\newcounter{optimhighlight}
\setcounter{optimhighlight}{\value{algorithm}}
\begin{algorithm}[t]
    \caption{\hfill Equation~\eqref{eq:optim-full}}
    \label{algo:capture}
    \vspace{.25em}
    \emph{Parameters:}
    \begin{itemize}
    \item Feasibility bounds $(\lambdamin, \lambdamax)$ and $(\omegaimin, \omegaimax)$
    \item Initial height $\hi$, its derivative $\hdi$, and target height $\hf$
    \item Discretization steps $\delta_1, \ldots, \delta_n$
    \end{itemize}
    \begin{subequations}
        \label{eq:optim-full}
        \begin{align}
            \minimize_{\bfphi \in \bbR^n}\ 
            & \sum_{j=1}^{n-1} \left[ 
            \frac{\varphi_{j+1}-\varphi_j}{\delta_j} - \frac{\varphi_j - \varphi_{j-1}}
            {\delta_{j-1}} \right]^2
            \label{eq:cvx-cost}
            \\
            \subjto\ 
            &
            \sum_{j=0}^{n-1} \frac{\delta_j}{\sqrt{\varphi_{j+1}} + \sqrt{\varphi_j}} 
            - \frac{\hi \sqrt{\varphi_n} + \hdi}{g} = 0
                    \label{eq:conv-cons-3d} \\
            &
            \omegaimin^2 \leq \varphi_n \leq \omegaimax^2 \label{eq:omega-i-3d} \\
            &
            \forall j < n,\ \lambdamin \delta_j \leq \varphi_{j+1} - \varphi_j \leq
            \lambdamax \delta_j \label{eq:optim-full-ineq} \\
            &
            \varphi_1 = \delta_0 g / \hf \label{eq:cvx-cons-last}
        \end{align}
    \end{subequations}
\end{algorithm}
\renewcommand{\thealgorithm}{\arabic{algorithm}}
\setcounter{algorithm}{\value{optimhighlight}}
\floatname{algorithm}{Algorithm}

\subsection{Computation and behavior of CoM capture trajectories}

Suppose that we know the solution $\bfvarphi^*$ to the capture problem, \emph{e.g.} computed using the solver from Section~\ref{sec:optim}. The first step to return from $s$ to time trajectories is to calculate the times $t_j = t(s_j)$ where the stiffness $\lambda$ switches from one value to the next.

Recall how the piecewise-constant values of $\lambda$ are readily computed from $\bfvarphi$ via Equation~\eqref{eq:lambda-omega-from-phi}. On an interval $[t_{j+1}, t_j)$ where $\lambda(t) = \lambda_j$ is constant, we can solve the differential equation $\sdd = \lambda_j s$ to obtain:
\begin{equation}
    \label{eq:s-of-t}
    s(t) = s_{j+1} \left[
        \cosh(x_j(t)) - \frac{\omega(s_{j+1})}{\sqrt{\lambda_j}} \sinh (x_j(t))
        \right]
\end{equation}
where $x_j(t) = \sqrt{\lambda_j} (t - t_{j+1})$.
Solving for the boundary condition $s(t_j) = s_j$ yields the next time $t_j = t(s_j)$:
\begin{equation}
    \label{eq:time-mapping}
    t(s_j) = t_{j+1} + \frac{1}{\sqrt{\lambda_j}} \log \left( \frac{\sqrt{\varphi_{j+1}} + \sqrt{\lambda_j} s_{j+1}}{\sqrt{\varphi_j} + \sqrt{\lambda_j} s_j}
    \right)
\end{equation}
By backward recursion, we can thus compute the time partition $0 = t_n < t_{n-1} < \ldots < t_1 < \infty$. The set of values $(\lambda_j, t_j)$ characterizes $\lambda(t)$ as a function of time.

Given both $s(t)$ from Equation~\eqref{eq:s-of-t} and $\omega(s)$ from Equation~\eqref{eq:lambda-omega-from-phi}, we can compute the CoP trajectory:
\begin{equation}
    \bfr(t) = \bfrf + (\bfri - \bfrf) \left(
        \frac{s(t) \omega(s(t))}{\sqrt{\varphi_n}}
        \right)^{\frac{\alpha}{1-\alpha}}
\end{equation}
Finally, once both components $\lambda(t)$ and $\bfr(t)$ of the capture input are known, we can compute the CoM capture trajectory by forward integration of the VHIP equation of motion~\eqref{eq:vhip}. Figure~\ref{fig:zero-step-behavior} shows various such trajectories for zero-step capture scenarios with the same initial state $\bfxinit$ but different contact locations.

\begin{figure}[t]
    \centering
    \includegraphics[height=2.8cm]{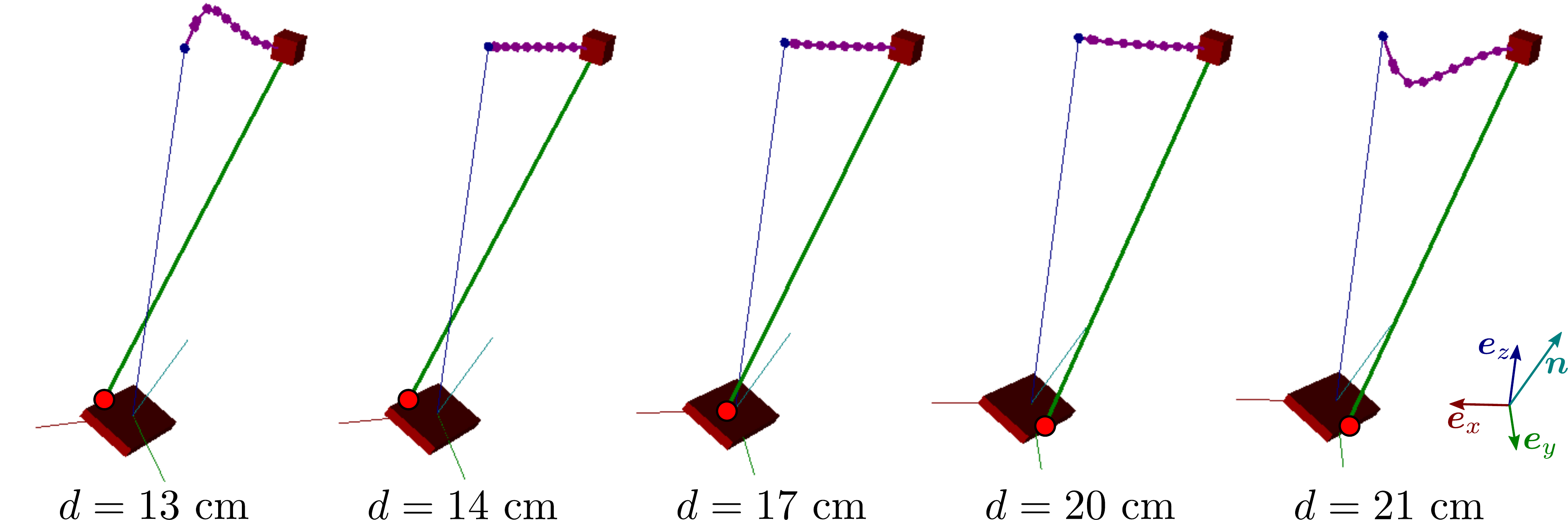}
    \caption{
        \textbf{Zero-step capture trajectories for different contact locations.} The horizontal distance $d$ from CoM to contact varies while the initial CoM position and velocity are kept constant. Red discs indicate the instantaneous CoP location $\bfri$. The regularization cost keeps the trajectory as close as possible to a LIP via CoP variations. When there is no linear solution, height variations are resorted to for additional acceleration or braking.
    }
    \label{fig:zero-step-behavior}
\end{figure}

\section{Optimization of capture problems}
\label{sec:optim}

While the capture problem~\eqref{eq:optim-full} belongs to the general class of nonconvex optimization, we can leverage its structural properties to solve it much faster than a generic nonconvex problem. This Section presents a dedicated solver that can solve capture problems within tens of \emph{microseconds}. Readers more interested in walking pattern generation can skip directly to Section~\ref{sec:one-step}.

In what follows, we assume that the reader is already familiar with common knowledge in numerical optimization, including the active-set method for quadratic programming (QP) and the sequential quadratic programming (SQP) method for nonlinear optimization. An overview of this background is provided for reference in Appendix~\ref{app:optim-background}.

\subsection{Problem reformulation}

The capture problem~\eqref{eq:optim-full} has a linear least squares cost, linear constraints and \emph{a single one-dimensional nonlinear equality constraint}. Its objective~\eqref{eq:cvx-cost} can be written $\left\|\bfJ \bfphi \right\|^2$ where the cost matrix $\bfJ$ is the $(n-1) \times n$ matrix given by:
\begin{equation*}
  \bfJ = \BIN -d_0 -d_1 &  d_1 & & &\\
							 d_1 & -d_1 -d_2  & d_2 &&\\
							 & & \ddots & & \\
							 && \hspace{-10pt}d_{n-2}  &-d_{n-2} -d_{n-1} & d_{n-1} \BOUT
\end{equation*}
with $d_j = \delta_j^{-1}$. Equation~\eqref{eq:conv-cons-3d} rewrites to $b(\bfphi) = 0$ with: 
\begin{equation}
    \label{eq:b-varphi}
    b(\bfphi) = \sum_{j=0}^{n-1} \dfrac{\delta_j}{\sqrt{\varphi_{j+1}} + \sqrt{\varphi_j}} - \frac{\hi \sqrt{\varphi_n} + \hdi}{g}
\end{equation}
where $\hi > 0$ and $\hdi \in \bbR$ are problem parameters, \emph{i.e.} constant during the optimization. Linear constraints~\eqref{eq:omega-i-3d}--\eqref{eq:cvx-cons-last} have the form:
\begin{equation}
    \bfl \leq \bfC \bfphi \leq \bfu
\end{equation}
where the two vectors $\bfl\in \bbR^{n+1}$ and $\bfu \in \bbR^{n+1}$ are also problem parameters, and $\bfC$ is the $(n+1) \times n$ matrix:
\begin{equation*}
  \bfC = \BIN \bfC_Z \\ \bfe_n^T \BOUT
  \text{ where }
	\bfC_Z = \BIN  1 &        &        &  \\
	              -1 &    1   &        &  \\
						       & \ddots & \ddots &  \\ 
						       &        &    -1  & 1\BOUT 
\end{equation*}
with $\bfe_n$ is the last column of the $n \times n $ identity matrix. Equality constraints are specified by letting $l_j = u_j$.

Solutions to the capture problem~\eqref{eq:optim-full} can be approximated by solving:
\begin{subequations}
\label{eq:transformed-problem}
\begin{align}
  \minimize_{\bfphi \in \bbR^n} & \ \frac{1}{2}\left\|\bfJ \bfphi \right\|^2 + \frac{\mu^2}{2} \left\| b(\bfphi) \right\|^2 \label{eq:least-squares-obj}\\
	\subjto & \ \bfl \leq \bfC \bfphi \leq \bfu \label{eq:linear-constraints}
\end{align}
\end{subequations}
which presents the advantage of having only linear constraints, and whose solution tends to the original solution as $\mu$ goes to infinity. This reformulation is reminiscent of penalty-based methods, where the penalty parameter $\mu$ would be adapted during successive iterations. Yet, a fixed parameter (typically $\mu = 10^6$) is enough to get a precise solution in practice, so that we do not need to adapt this parameter.

\subsection{Application of an SQP approach}
\label{seq:sqp-approach}

We apply the SQP method to the reformulation~\eqref{eq:transformed-problem}. Let us denote by $f(\bfvarphi)$ the objective~\eqref{eq:least-squares-obj} of the problem and $\bfj \defeq \bfnabla_{\bfphi} b$ the gradient of the nonlinear constraint. The Lagrangian of problem~\eqref{eq:transformed-problem} is 
\begin{equation*}
    \calL(\bfvarphi, \bflambda_-, \bflambda_+)  = f(\bfvarphi) + \bflambda_-^T (\bfl - \bfC \bfphi) + \bflambda_+^T (\bfC \bfphi - \bfu)
\end{equation*}
with $\bflambda_-, \bflambda_+ \in \bbR^{n+1}$ the corresponding Lagrange multipliers.

Let us index by $\square_k$ the value of any quantity at iteration $k$ of the SQP method. The Hessian matrix at iteration $k$ is then:
\begin{align}
\bfnabla^2_{\bfphi\bfphi} \calL_k & = \bfJ^T \bfJ + \mu^2 \bfj_k \bfj_k^T + b_k \bfnabla^2_{\bfphi\bfphi}b_k
\end{align}
We adopt the Gauss-Newton approximation $\bfnabla^2_{\bfphi\bfphi} \calL_k \approx \bfJ^T \bfJ + \mu^2 \bfj_k \bfj_k^T$, a classical approach for nonlinear least squares. It is particularly well-suited to the present case, as $\bfnabla^2_{\bfphi\bfphi} \calL$ is exactly equal to $\bfJ^T \bfJ + \mu^2 \bfj \bfj^T$ when the boundedness condition $b=0$ is satisfied. Under this approximation, the problem for one iteration of the SQP method becomes:
\begin{subequations}
\label{eq:LSI}
\begin{align}
  \minimize_{\bfp \in \bbR^n} & \ \frac{1}{2} \left\|\bfJ\bfp + \bfJ\bfphi_k\right\|^2 + \frac{\mu^2}{2}\left\|\bfj_k^T \bfp + b_k\right\|^2 \label{eq:subproblem-obj}\\
	\subjto & \ \bfl'_k \leq \bfC \bfp \leq \bfu'_k \label{eq:subproblem-constr}
\end{align}
\end{subequations}
where $\bfl'_k \defeq \bfl - \bfC \bfphi_k$ and $\bfu'_k \defeq \bfu - \bfC \bfphi_k$.
If $\bfphi_k$ satisfies the linear constraints~\eqref{eq:linear-constraints}, then $\bfp = 0$ is a feasible point for the problem~\eqref{eq:LSI}. This problem is a linear least squares with inequality constraints (LSI), a particular case of QP, that we can solve using the active-set method (Algorithm~\ref{alg:active-set}).
Adopting $\bfd$ and $j$ for the step and iteration number of the QP (keeping $\bfp$ and $k$ for the SQP), an iteration of the active-set method solves in this case:
\begin{subequations}
\begin{align}
    \minimize_{\bfd \in \bbR^n} & \ \frac{1}{2} \left\|\bfJ (\bfd + \bfp_j + \bfphi_k)\right\|^2 + \frac{\mu^2}{2}\left\|\bfj_k^T (\bfd + \bfp_j) + b_k\right\|^2 \\
      \subjto & \ \bfC_{\calW_j} \bfd = 0
\end{align}
\end{subequations}
with $\calW_j$ the set of active constraints at the current iteration $j$ and $\bfC_{\calW_j}$ the corresponding matrix.

This can be solved in two steps using the nullspace approach. First, compute a matrix $\bfN_{\calW} \in \bbR^{n \times n-r}$ whose columns form a basis of the nullspace of $\bfC_{\calW}$, $r$ being the rank of $\bfC_{\calW}$. The vector $\bfd$ is then solution of the problem if and only if $\bfd = \bfN_{\calW}\bfz$ for some $\bfz\in \bbR^{n-r}$. The problem can thus be rewritten as an unconstrained least squares:
\begin{equation}
  \minimize_{\bfz \in \bbR^{n-r}} \quad \frac{1}{2} \left\|\BIN \mu \bfj^T \\ \bfJ \BOUT \bfN_{\calW} \bfz + \BIN \mu(\bfj^T \bfp_j + f) \\ \bfJ(\bfp_j + \bfphi) \BOUT \right\|^2
\end{equation}
Second, solve this unconstrained problem: taking $\bfT$ and $\bfu$ such that the above objective writes $\frac{1}{2} \left\|\bfT z + \bfu\right\|^2$, compute the QR decomposition $\bfT = \bfQ \bfR$, and solve $\bfQ \bfR \bfz = - \bfu$. The latter is equivalent to $\bfz = -\bfR^{-1} \bfQ^T \bfu$ if $\bfR$ has full rank~\cite[Chapter~10]{nocedal:book:2006}. Both of these steps can be significantly tailored to the case of capture problems.

\subsection{Tailored operations}
\label{subsec:tailored}

In the SQP, most of the time is spent in solving the underlying LSI: the computation of $\bfN_{\calW}$, the post-multiplication by $\bfN_{\calW}$ to obtain $\bfT$, the QR decomposition of $\bfT$ and the computation of the Lagrange multipliers are the main operations, performed each roughly in $O(n^3)$~\cite{golub:book:1996}, at least for the first iteration of each LSI.\footnote{
    We could refine these estimates by taking into account the number of active constraints. Note also that subsequent LSI iterations can perform some of these operations in $O(n^2)$.
} We can reduce this complexity to at most $O(n^2)$ for capture problems:

\begin{itemize}
    \item The matrix $\bfN_{\calW}$ does not need to be computed explicitly. Rather, the cost matrix $\bfT$ of the unconstrained least squares problem can be built directly in $O(n)$ operations by taking advantage of the structure of $\bfC_{\calW}$ and $\bfJ$ (Appendix~\ref{app:build-T})
    \item Lagrange multipliers needed to check KKT conditions can be computed in $O(n^2)$ by taking advantage of the structure of $\bfC_{\calW}$ (Appendix~\ref{app:lagrange-multipliers})
    \item The QR decomposition of $\bfT$ can be carried out in $O(n^2)$ by leveraging the tridiagonal structure of $\bfJ_{\calW}$ (Appendix~\ref{app:qr-T})
\end{itemize}

The last point to consider is finding an initial pair $(\bfphi_0,\bflambda_0)$ for the SQP. While in a classical SQP this is done through a so-called \emph{Phase I} which can be almost as costly as running the main loop of the algorithm itself, we can leverage the geometry of our constraints to get such a pair in $O(n)$ (Appendix~\ref{app:initial-point}). 

\subsection{Numerical and algorithmic considerations}

The implementation of a general-purpose QP or SQP solver is an extensive work due to the numerous numerical difficulties that can arise in practice: active-set methods need to perform a careful selection of their active constraints in order to keep the corresponding matrix well conditioned, while SQPs require several refinements, some of which imply solving additional QPs at each iteration~\cite{nocedal:book:2006}. 
While the tailored operations we presented reduce the theoretical complexity w.r.t. general-purpose solvers, there are also a number of features of Problem~(\ref{eq:transformed-problem}) that allow us to stick with a simple, textbook implementation, and contribute to the general speed-up.

On the QP side, the matrix $\bfC_Z$ is always full rank and well conditioned, while the last row $\bfe_n^T$ of $\bfC$ is a linear combination of \emph{all} rows from $\bfC_Z$ ($\bfe_n = \bfC_Z^T \bm{1}_n$). As a consequence, all matrices $\bfC_\calW$ are full rank and well conditioned, save for the case where all $n+1$ constraints are active. This case can be easily detected and avoided.\footnote{
    With the notations of Appendix~\ref{app:initial-point}, this can only happen if $a^-$ or $a^+$ is equal to $0$ or $1$. A workaround is then to slightly perturb $\omegaimin$ or $\omegaimax$.
} It is thus safe to use a basic active-set scheme.

All QR decompositions are performed on matrices with rank deficiency of at most one. As a consequence, it is not necessary to use more involved column-pivoting algorithms, and the rank deficiency can be detected by simply monitoring the bottom-right element of the triangular factor. While we don't prove that the matrices $\bfJ_\calW$ are well conditioned, we verified this assertion for $n \leq 20$ in a systematic way. Even for large values of $\mu$, the QR decomposition of $\bfT$ is stable as the row with largest norm appears first~\cite[p. 169]{bjorck:book:1996}.

\subsection{Pre-computation of QR decompositions}

It is important to note that the matrices $\bfJ$ and $\bfC$ only depend on the problem size $n$ and partition $s_0, \ldots, s_n$, which are the same across all capture problems that we solve in this work. If $n$ is small enough (say $n \leq 20$), we can precompute and store the QR decompositions of \emph{all} possible $\bfJ_\calW \defeq \bfJ \bfN_{\calW}$. Note that there are $2^{n+1}-1$ different sets $\calW$, including all combinations of up to $n$ active constraints among $n+1$. These pre-computations can be done in a reasonable amount of time thanks to our specialized QR decompositions, and result in even faster resolution times.

On the SQP side, the odds are very favorable: constraints are linear and the Gauss-Newton approach offers a good approximation of the Hessian matrix. As a consequence, we observed that the method takes full steps
$98.5\%$ of the time in practice, and converges in very few iterations (four on average).
Since we start from a feasible point, all subsequent iterates are guaranteed to be feasible and the line search needs only monitor the objective function, in an unconstrained-optimization fashion.

\subsection{Performance comparison}
\label{subsec:optim-perf}

\begin{table}
    \caption{
        Computation times over 20000 sample problems from walking pattern generation.
        Averages and standard deviations in $\mu$s.
    }
		\vspace{-5pt}
    \label{table:times}
    \centering
    \setlength\tabcolsep{4pt}  
    \begin{tabular}{lrrrr}
        \textbf{Solver}    & $n = 10$        & $n = 15$         & $n = 20$             & $n = 50$ \\
        \hline
        IPOPT\footnotemark & $7.1 \times 10^3$ & $9.4 \times 10^3$ & $1.1 \times 10^4$ & $2.2 \times 10^4$ \\
        SQP + LSSOL        & $86 \pm 60$     & $130 \pm 86$     & $220 \pm 160$        & $1700 \pm 1700$ \\
        SQP + cLS          & $22 \pm 12$     & $33 \pm 18$      & $54 \pm 41$          & $\mathbf{210 \pm 180}$ \\
        SQP + cLS + pre.   & $\mathbf{18 \pm 10}$     & $\mathbf{25 \pm 14}$      & $\mathbf{35 \pm 22}$          & -- \\
        \hline
    \end{tabular}
		\vspace{-5pt}
\end{table}

\footnotetext{
    We only report averages for IPOPT computation times as they lie on a different scale. These averages are higher than those reported in~\cite{caron2018icra} because we evaluate both feasible and unfeasible problems (for reasons made clear in the next section), while all random initial conditions in~\cite{caron2018icra} were zero-step capturable. For $n=10$ and projecting performance statistics on feasible problems only, IPOPT's computation times decrease to $1600 \pm 790$ $\mu$s.
}

In~\cite{caron2018icra}, Problem~\eqref{eq:optim-full} was solved with the state-of-the-art solver IPOPT~\cite{wachter2006springer}, which is written in Fortran and C. We compare its performances with our tailored SQP approach, implemented in C++.\footnote{
    \url{https://github.com/jrl-umi3218/CaptureProblemSolver}
} Taking $\mu = 10^6$, the solutions returned by both methods are numerically equivalent (within $10^{-7}$ of one another, and $\left|b(\bfphi)\right| \approx 10^{-8}$ in both cases). 
Computation times over representative problems produced during walking pattern generation are reported in Table~\ref{table:times}, where our approach is denoted \emph{SQP + cLS} (custom least squares), and the abbreviation \emph{pre.} denotes the use of QR pre-computations for $\bfJ_\calW$.
In practice, we work with $n=10$, for which our solver is 300--400 times faster than the generic solver IPOPT.

Computation times for QR pre-computations range from 2~ms for $n=10$, 100~ms for $n=15$, to 4.9~s for $n=20$. This is not limiting in practice, as these computations are performed only once at startup. The limit rather lies with memory consumption, which follows an exponential law ranging from 2~MB for $n=10$ to 5~GB for $n=20$. 

To break down how much of the speed-up is due to the problem reformulation and how much is due to our custom least squares implementation, we also test our SQP method using the state-of-the-art least squares solver LSSOL~\cite{gill:tech:1986}. This variant is denoted by \emph{SQP + LSSOL}.
For $n \geq 25$, SQP + LSSOL starts to fail on some problems, with a failure rate of roughly $25\%$ for $n = 50$. This suggests that the least squares component of our solver is more robust. There are two plausible explanations for this. First, LSSOL assumes all matrices are dense, while we leverage sparsity patterns of $\bfJ$ and $\bfC$. Second, LSSOL treats all coefficients as floating-point numbers, while knowing that the elements of $\bfC$ are exactly $1$ or $-1$ allows us to carry out exact computations (most notably for the nullspace and pseudoinverse of $\bfC_\calW$).

\section{Walking pattern generation}
\label{sec:one-step}

While zero-step capturability enables push recovery, \emph{one-step} capturability is the minimum price to pay for walking. 

\subsection{Time-varying CoP strategy}

In a one-step capture setting, the contact area switches instantaneously at a given instant $\sc \in (0, 1)$: 
\begin{align}
    \calC(s) & = \begin{cases}
        \calCi & \text{for } \sc < s \leq 1 \\ 
        \calCf & \text{for } 0 < s \leq \sc
        \end{cases}
\end{align}
Due to this discontinuity, we cannot adopt the same line-segment CoP strategy that we used in Section~\ref{sec:zero-step} for balance control. Let us then adopt a piecewise-constant CoP trajectory:
\begin{align}
    \label{eq:cons-cop-strat}
    \bfr(s) & = \begin{cases}
        \bfri & \text{for } \sc < s \leq 1 \\
        \bfrf & \text{for } 0 < s \leq \sc
        \end{cases}
\end{align}
This choice yields the following boundedness condition~\eqref{eq:boundedness-s}:
\begin{equation}
    -\bfg \int_0^1 \frac{\dd{s}}{\omega(s)} = \omegai (\bfci -\bfri) + (\bfri - \bfrf) \sc \omega(\sc) + \bfcdi
\end{equation}
In terms of our optimization variable $\bfvarphi$, the right-hand side of this equation is:
\begin{equation}
     (\bfci - \bfri) \sqrt{\varphi_n} + (\bfri - \bfrf) \sqrt{\varphi(\sc)} + \bfcdi
\end{equation}
where $\varphi(\sc)$ is a linear combination of $\varphi_j$ and $\varphi_{j+1}$ when $\sc \in [s_j, s_{j+1}]$ from Equation~\eqref{eq:varphi-s}.

Let us define an external parameter $\alpha \in (0, 1)$. This time, we parameterize the contact switch $\sc$ of the CoP trajectory by $\alpha$ as follows:
\begin{equation}
    \label{eq:sc-choice}
    \sc \text{ is the scalar s.t. } \sqrt{\varphi(\sc)} = \alpha \sqrt{\varphi_n}
\end{equation}
Under this assumption, the right-hand side of the boundedness condition then simplifies to:
\begin{equation}
\label{eq:reduc-alpha}
(\bfci - \bfr_\alpha) \sqrt{\varphi_n} + \bfcdi
\end{equation}
where $\bfr_\alpha \defeq \alpha \bfrf + (1 - \alpha) \bfri$ is the \emph{equivalent CoP}~\cite{koolen2012ijrr} of the time-varying CoP trajectory. In the LIP with constant $\omega$, Equation~\eqref{eq:sc-choice} rewrites to $\alpha = e^{-\omega t_\csubscript}$,\footnote{
    This expression corresponds to the limit of the scalar $w'$ when $\Delta t_1' \to \infty$ in Equation~(25) of~\cite{koolen2012ijrr}. It also plays an important role in the step timing adjustment method from~\cite{khadiv2016humanoids}.
} which shows how $\alpha$ parameterizes the time of contact switch. This observation still holds for the VHIP, although we don't know the mapping $\tc(\alpha)$ in closed-form any more, and will prove useful thereafter.

Taking the dot product of this equation with the normal $\bfni$ of the initial contact $\calCi$ yields: 
\begin{align}
    \label{eq:reduc-bz}
    \int_0^1 \frac{\dd{s}}{\omega(s)} & = \frac{\halpha \sqrt{\varphi_n} + \hdi}{g} &
    \halpha & \defeq \frac{(\bfci - \bfr_\alpha) \cdot \bfni}{(\bfe_z \cdot \bfni)}
\end{align}
This constraint is exactly of the form~\eqref{eq:conv-cons-3d} of the capture problem~\eqref{eq:optim-full}, where the parameter $\hi$ has simply been replaced by $\halpha$. 

The horizontal components of Equation~\eqref{eq:reduc-alpha} give:
\begin{align}
    \label{eq:reduc-bxy}
    \bfri^{xy} & = \bfrf^{xy} + \frac{1}{1 - \alpha} \left[\bfci^{xy} + \frac{\bfcdi^{xy}}{\omegai} - \bfrf^{xy} \right]
\end{align}
This expression is identical to Equation~\eqref{eq:alpha-rixy}. Following the same steps~\eqref{eq:below-alpha-rixy}--\eqref{eq:zs-last-step} as in the zero-step setting, we therefore conclude that capture inputs under the CoP strategy~\eqref{eq:cons-cop-strat}--\eqref{eq:sc-choice} are characterized by the 1D capturability condition (Property~\ref{prop:1d-capturability}). In particular, we can compute them using our capture problem solver.


\subsection{External optimization over contact switching time}

The contact switching time $\tc(\alpha)$ plays a crucial role in one-step capture trajectories.
Given the solution $\bfvarphi_\alpha$ to the capture problem parameterized by $\alpha$, we can compute it as follows. First, compute $\sc$ using Equation~\eqref{eq:lambda-omega-from-phi}:
\begin{equation}
    \sc = \sqrt{s_j^2 + \frac{\alpha^2 \varphi_{\alpha,n} - \varphi_{\alpha,j}}{\lambda_{\alpha,j}}}
\end{equation}
where $\varphi_\alpha(\sc) = \alpha^2 \varphi_{\alpha,n} \in [\varphi_{\alpha,j}, \varphi_{\alpha,j+1}]$. Then, apply the mapping $\tc = t(\sc)$ provided by Equation~\eqref{eq:time-mapping}.

While this mapping is straightforward to compute numerically, its high nonlinearity suggests that introducing time constraints into capture problems would radically affect the design of the dedicated solver from Section~\ref{sec:optim}.
Fortunately, computation times achieved by this solver allow us to solve \emph{hundreds} of capture problems per control cycle. We therefore choose to optimize jointly over $\bfvarphi$ and $\alpha$ using a two-level decomposition: 
an external optimization over $\alpha \in (0, 1)$, wrapping an internal optimization where $\alpha$ is fixed and the existing capture problem solver is called to compute $\bfvarphi_\alpha$.

A straightforward way to carry out the external optimization is to test for a large number of values $\alpha \in (0, 1)$. This approach is however inefficient in practice as it sends a significant amount of unfeasible problems to the internal optimization. The main cause for such unfeasibility is that whole intervals of values for $\alpha$ result in $\omegaimin(\alpha) > \omegaimax(\alpha)$, for which the CoP constraint~\eqref{eq:omega-i-3d} is obviously unfeasible. To avoid this, we propose in Appendix~\ref{app:feasible-intervals} an algorithm to pre-compute the set of intervals $[\alphamin, \alphamax] \subset (0, 1)$ on which it is guaranteed that $\omegaimin \leq \omegaimax$.

During bipedal walking, contact switches can only be realized after the free foot has completed its swing trajectory from a previous to a new contact. We therefore need to make sure that $\tc(\alpha)$ is greater than the remaining duration $\tswing$ of the swing foot motion. This gives us an additional constraint to be enforced by the external optimization:
\begin{equation}
    \label{eq:tc-ineq}
    \tc(\alpha) \geq \tswing
\end{equation}
where each evaluation of $\tc(\alpha)$ costs the resolution of a full capture problem. We find such solutions by sampling $n_\alpha$ values of $\alpha$ per feasible interval $[\alphamin, \alphamax]$. This approach provides automatic step timings, and tends to output pretty dynamic gaits.

Alternatively, when step timings are provided one may want to enforce an equality constraint:
\begin{equation}
    \label{eq:tc-eq}
    \tc(\alpha) = \tswing
\end{equation}
In this case, the external optimization implements an approximate gradient descent, calling the capture problem solver multiple times to evaluate $\nabla_\alpha \tc$ and  converge to a local optimum on each feasible interval $[\alphamin, \alphamax]$. Although we don't prove it formally, we observed in practice that the mapping $\tc(\alpha)$ seems to be always monotonic (as in the LIP) and the optimum thus found global on its interval.

\subsection{Pattern generation from capture trajectories}

We can generate walking from capture trajectories via a two-state strategy~\cite{caron2017iros, sugihara2017iros} combining zero-step and one-step capture problems. The robot starts in a double-support posture, and computes a one-step capture trajectory to transfer its CoM to its first support foot.
\begin{itemize}
    \item \textbf{Single support phase:} swing foot tracks its pre-defined trajectory while the CoM follows a one-step capture trajectory updated in a model predictive control fashion. The phase ends at swing foot touchdown.
    \item \textbf{Double support phase:} the CoM follows a zero-step capture trajectory towards the target contact. Meanwhile, the capture problem solver is called to compute a one-step capture trajectory towards the next contact. The phase ends as soon as such a trajectory is found.
\end{itemize}
The behavior realized by this state machine is conservative: after touchdown, the robot uses its double support phase to slow down until a next one-step capture trajectory is found. If the next contact is not one-step capturable, the robot stops walking and balances in place. Otherwise, walking continues toward the next pair of footsteps. 

A limitation of this walking pattern generator is that double-support phases do not include a continuous CoP transfer from one contact to the next. The one exception to this is when both contacts are coplanar, in which case we can use their convex hull as halfspace representation $(\bfH, \bfp)$ in Equation~\eqref{eq:ineq-omegai} and thus apply the line-segment CoP strategy from Section~\ref{sec:tv-cop}. In practice, we let CoP discontinuities occur in the walking pattern during double support and rely on the underlying robot stabilizer to handle them (see Section~\ref{sec:dyn}).

\section{Simulations}
\label{sec:simus}

We validated this walking pattern generator in two successive implementations. First, we evaluate the kinematics of whole-body tracking of the VHIP reference in a prototyping environment. Then, we implement our method as a full-fledged C++ controller and evaluate its performance compared to the state of the art in two dynamic simulators.

\subsection{Kinematics}
\label{sec:kin}

\begin{figure*}[t]
    \centering
    \subfloat[Elliptic staircase]{
        \includegraphics[height=4.8cm]{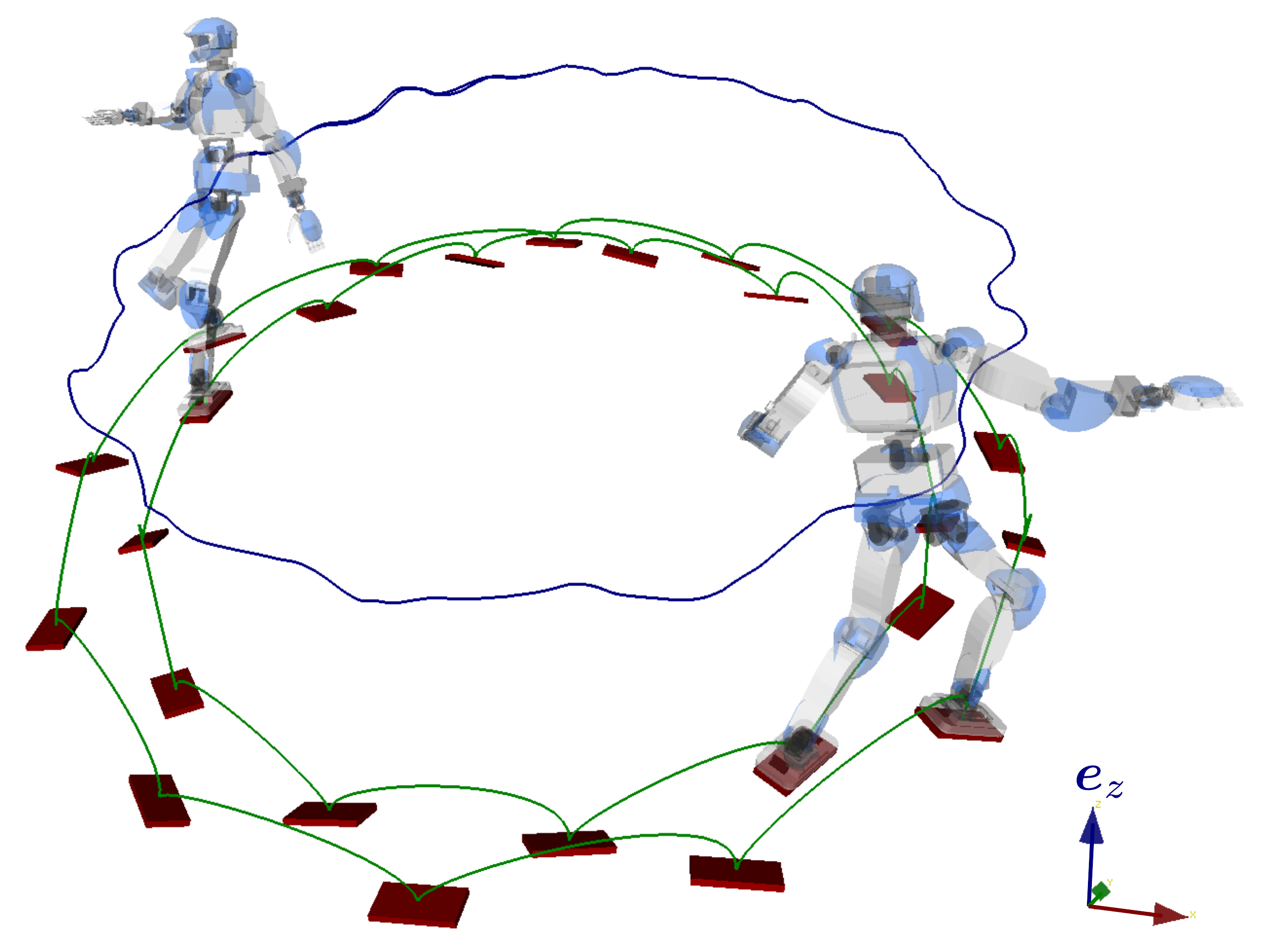}
        \label{fig:sim-elliptic}}
    \subfloat[Regular staircase]{
        \includegraphics[height=4.8cm]{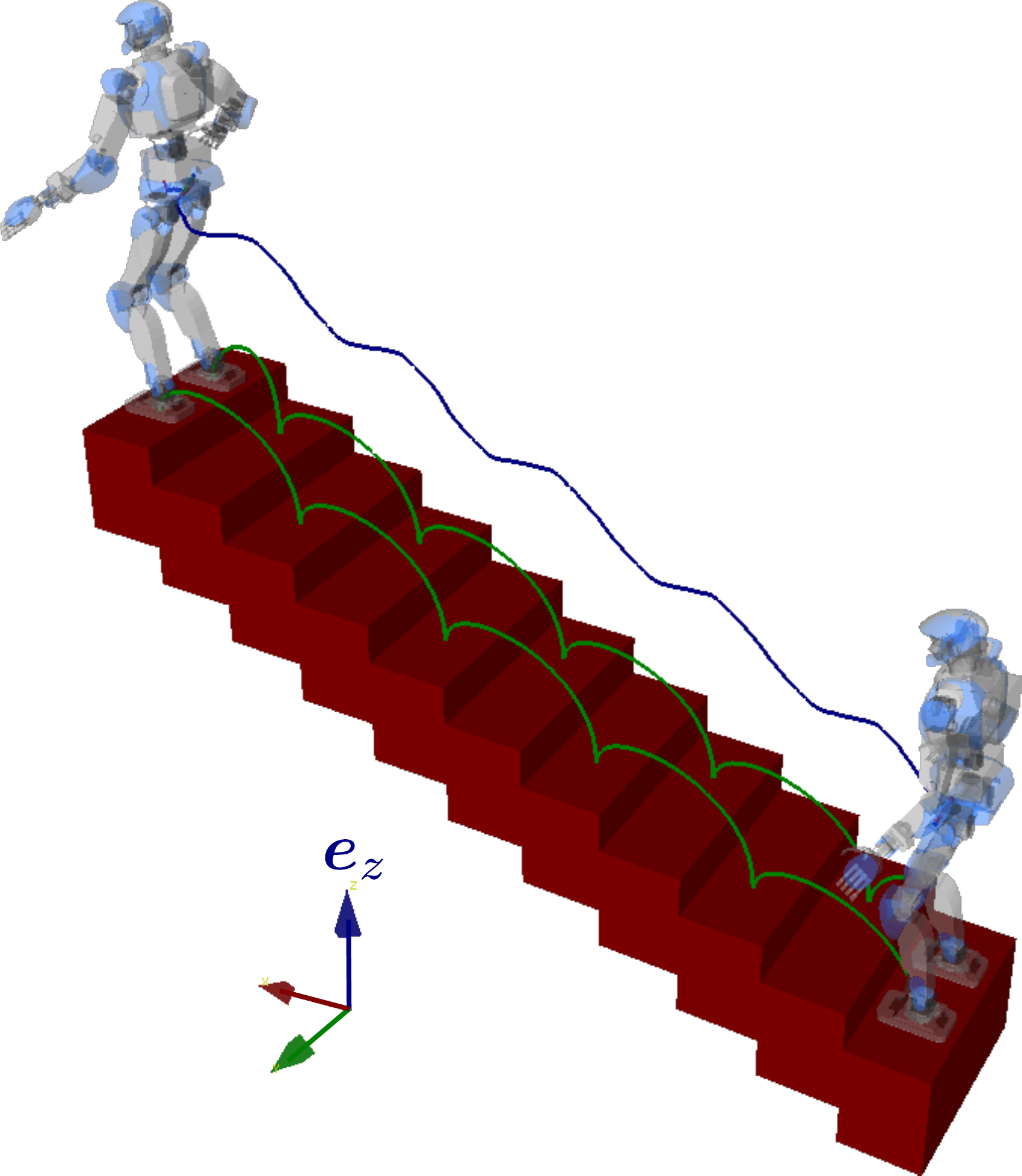}
        \label{fig:sim-regular}}
    \hspace{0.3cm}
    \subfloat[Aircraft factory]{
        \includegraphics[height=4.8cm]{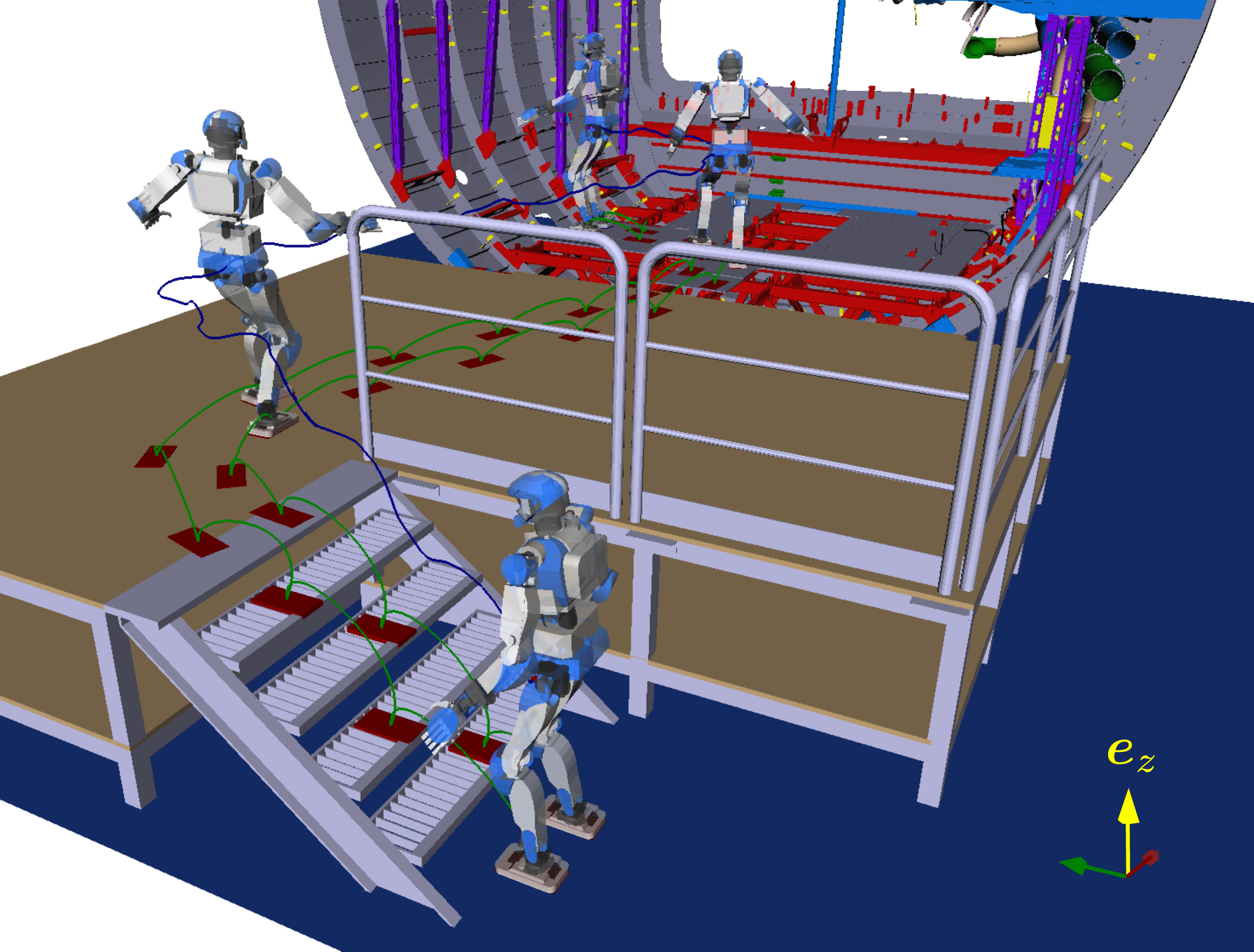}
        \label{fig:sim-aircraft}}
    \caption{
        \textbf{Walking pattern generation:} whole-body tracking in three scenarios. The elliptic staircase with randomly-tilted footsteps (a) tests the ability to walk over rough terrains, \emph{i.e.} to adapt to both 3D translation and 3D orientation variations between contacts. The regular staircase (b) has 15-cm high steps; it assesses the behavior of the solution when contacts are close to each other and collision avoidance becomes more stringent. Finally, the aircraft scenario (c) provides a real-life use case where the environment combines flat floors, staircases with 18.5-cm high steps and uneven-ground areas (inside the fuselage). In all three figures, blue and green trajectories respectively correspond to center-of-mass and swing-foot trajectories.}
    \vspace{1.25ex}
    \label{fig:sim}
\end{figure*}

We implemented our walking pattern generation method in the \emph{pymanoid}\footnote{
    \url{https://github.com/stephane-caron/pymanoid}
} prototyping environment. Whole-body inverse kinematics is carried out using a standard quadratic-programming formulation (see \emph{e.g.}~\cite[Section~1]{escande2014ijrr} for a survey). Tracking of the inverted pendulum and swing foot trajectories is realized by the following set of tasks:

\begin{table}[h]
    \centering
    \begin{tabular}{lll}
    \textbf{Task group} & \textbf{Task} & \textbf{Weight} \\
    \hline
    Foot tracking & Support foot & $1$ \\
    Foot tracking & Swing foot & $10^{-3}$ \\
    VHIP tracking & Center of mass & $1 \times 10^{-2}$ \\
    VHIP tracking & Min. ang. mom. variations & $1 \times 10^{-4}$ \\
    Regularization & Keep upright chest & $1 \times 10^{-4}$ \\
    Regularization & Min. shoulder extension & $1 \times 10^{-5}$ \\
    Regularization & Min. upper-body velocity & $5 \times 10^{-6}$ \\
    Regularization & Reference upright posture & $1 \times 10^{-6}$ \\
    \hline
    \end{tabular}
\end{table}

We considered the three scenarios depicted in Figure~\ref{fig:sim}. The first one, an elliptic staircase with randomly-tilted footsteps~(Figure~\ref{fig:sim-elliptic}), tests the ability to adapt to general uneven terrains. It illustrates the main advantage of the capture problem formulation: it provides real-time model predictive control with constraint saturation over rough terrains. Existing methods either enforce constraints on regular terrains such as floors or stairs~\cite{wieber2006humanoids, brasseur2015humanoids}, or do not take constraint saturation into account~\cite{ hopkins2014humanoids, englsberger2015tro}, or compute costly nonconvex optimizations offline~\cite{dai2014humanoids, carpentier2016icra} (see Section~\ref{sec:discussion} for a complete discussion).

We also consider a regular staircase with $15$-cm high steps (Figure~\ref{fig:sim-regular}), and the real-life scenario provided by Airbus Group depicted in Figure~\ref{fig:sim-aircraft}. It consists of a 1:1 scale model of an A350 aircraft under construction in a factory environment. To reach its desired workspace configuration, the humanoid has to walk up an industrial-grade staircase (step height $18.5$~cm, except the last one which is $14.5$~cm), then across a flat floor area and finally inside the fuselage where the ground consists of temporary wooden slabs. All three walking patterns are depicted in Figure~\ref{fig:sim} and in the accompanying video.

We use the same set of parameters on all scenarios. For capture problems, we choose $n = 10$ discretization steps with a partition $s_i = i / n$ and set $\hf = 0.8$~m, which is a suitable CoM height for HRP-4 at rest with an extended leg. The external optimization for one-step problems samples $n_\alpha = 5$ values of $\alpha$ per feasibility interval, which is sufficient to find solutions that satisfy the inequality constraint~\eqref{eq:tc-ineq} on contact switching time. For the VHIP stiffness feasibility, we set $\lambdamin = 0.1 g$ and $\lambdamax = 2 g$.

At each control cycle, the capture problem solver is called on both zero-step and one-step capture problems. On a consumer laptop computer, zero-step problems were solved in $0.38 \pm 0.13$~ms while one-step ones were solved in $2.4 \pm 1.1$~ms (average and standard deviations over 10,000 control cycles). Note that these computation times reflect both calls to the C++ solver (Section~\ref{sec:optim}) and the external optimization over $\alpha$ implemented in the prototyping environment in Python.

We release the code of this prototype as open source software.\footnote{
    \url{https://github.com/stephane-caron/capture-walking-pg}
}

%

\subsection{Dynamics}
\label{sec:dyn}

To test our method in dynamic simulations and compare it to the existing, we extend the C++ LIP-based stair climbing controller from~\cite{caron2018hal}. This controller consists of two main components: a LIP-based pattern generator by model predictive control (\emph{i.e.} the pattern is produced online in receding horizon) and a stabilizer based on whole-body admittance control. We implement our solution as a second pattern generator in this controller, using the same stabilizer in both cases for comparison.

The planar CoM constraint of the LIP is applicable to small step heights (\emph{e.g.} Kajita \emph{et al.} applied it to 10-cm high steps in~\cite{kajita2003icra}) but can become problematic for higher steps. From experiments on the HRP-4 robot,~\cite{caron2018hal} rather chose to decompose the constraint in two horizontal plane segments, introducing the vertical height variation at toe liftoff. These segments could be further decomposed for better adaption to terrain, but as of today there is no known algorithm for this. Using a VHIP-based pattern generator addresses the question by removing the specification of plane segments altogether.

Thanks to the mc\_rtc control framework,\footnote{
    Developed in the course of the COMANOID project, this control framework is available at \url{https://gite.lirmm.fr/multi-contact/mc\_rtc} upon request and will be released soon.
} we validated the feasibility of our method in two different dynamic simulators: 
\begin{itemize}
    \item \textbf{Choreonoid}\footnote{
        \url{http://choreonoid.org/en/}
        }, the most reliable one at our disposal, which implements the projected Gauss-Seidel method~\cite{nakaoka2007rsj} and runs in real time.
    \item \textbf{V-REP}\footnote{
        \url{http://coppeliarobotics.com/}
    } with the Newton dynamics engine. For a realistic behavior, we set simulations to a time step of 5~ms with 1~ppf, in which case they were not real time.  
\end{itemize}
The same set of controller parameters, and thus the same walking patterns, were applied in both simulators.

Figures~\ref{fig:choreonoid} and~\ref{fig:vrep} compare LIP-based and VHIP-based controller simulations at different time instants (full simulation recordings are shown in the accompanying video). The main difference between the two lies in knee flexion-extension (Figure~\ref{fig:knees}): thanks to its reference height going up after the first step, the VHIP pattern brings the robot to knee extension while the LIP one bends down then raises again (Figure~\ref{fig:com-height}). In the horizontal plane, DCM-ZMP tracking looks similar to a regular capture-point walking controller (Figure~\ref{fig:lateral-vhip}) since the variations of $\omega(t)$ required to lift the CoM up are very small: in this stair climbing scenario, $\omega$ ranges from 3.536~Hz to 3.550~Hz (Figure~\ref{fig:omega-variations}), as opposed \emph{e.g.} to the elliptic staircase scenario (Figure~\ref{fig:sim-elliptic}) where $\omega$ ranges from 3.4~Hz to 3.8~Hz.

We release the code of this controller as open source software.\footnote{
    \url{https://github.com/stephane-caron/capture\_walking\_controller}
}

\section{Discussion and future work}
\label{sec:discussion}

While the present study is coming to an end, our understanding of the variable-height inverted pendulum (VHIP) model is, hopefully, only beginning to unfold. What did we understand so far? First, that capturability of the VHIP is characterized by three properties of its two inputs: their feasibility, asymptotic convergence, and the boundedness condition. These properties can be cast into an optimization problem, the \emph{capture problem}, that we can solve in tens of microseconds. These fast solver times allow us to add an external optimization solving for other nonlinearities such as step timings.

\begin{figure}[t]
    \centering
    \includegraphics[width=\columnwidth]{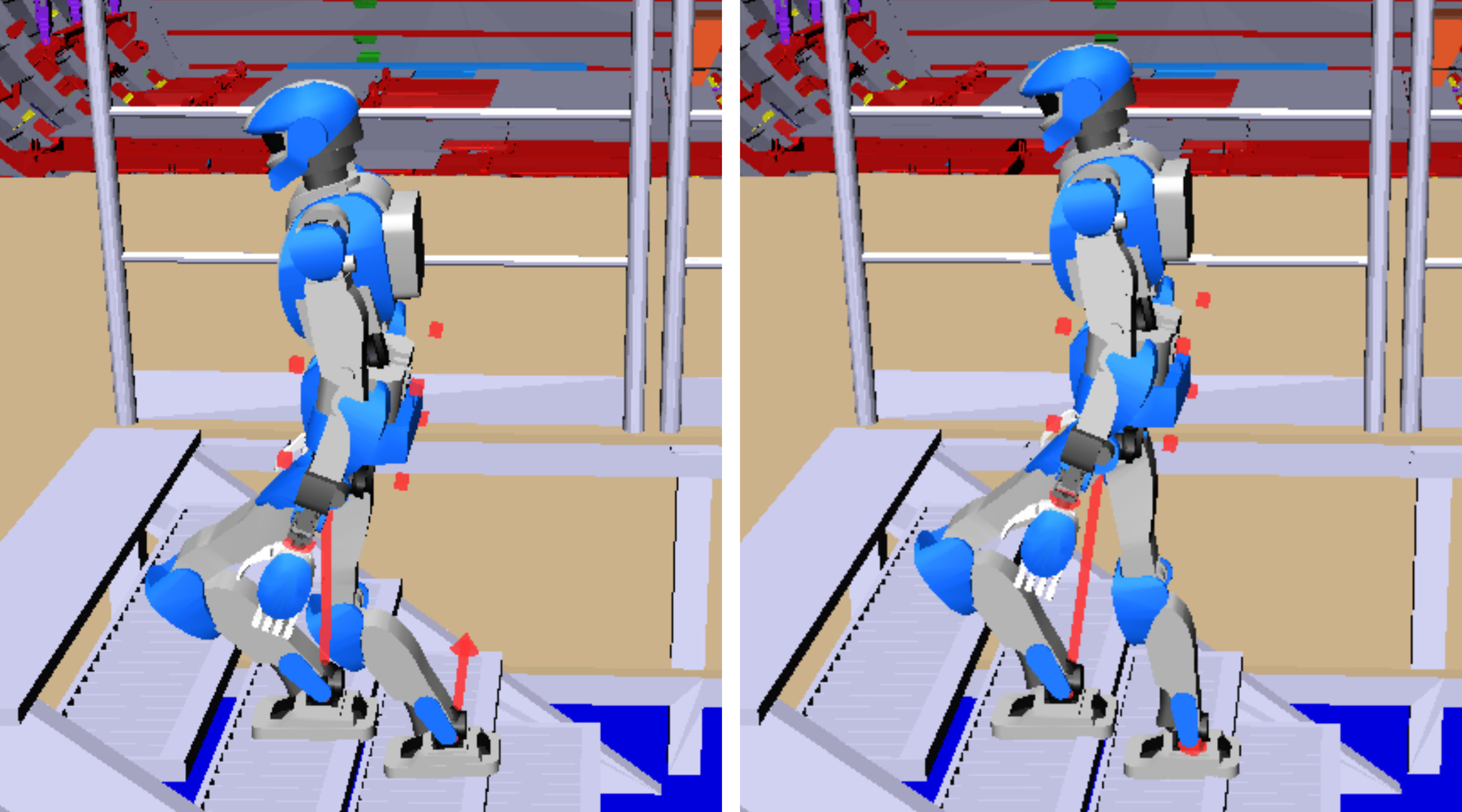}
    \caption{
        \textbf{Dynamic simulations in Choreonoid:} snapshots at the maximum knee flexion time instant. \textit{Left:}~LIP-based pattern generation. \textit{Right:}~VHIP-based pattern generation. The former causes HRP-4 to bend its knees more, increasing the risk of shank collision with the staircase. We can also observe the more abrupt CoP transfer of the latter, which is specific to our method.
    }
    \label{fig:choreonoid}
\end{figure}

Our overall discussion draws numerous connections with the existing literature. To start with, the exponential dichotomy of the time-varying inverted pendulum was proposed in 2004 by J. Hauser \emph{et al.}~\cite{hauser2004cdc} to address a question of motorcycle balance. Its application to the linear inverted pendulum can be found in the motion generation framework of the Honda ASIMO humanoid~\cite{takenaka2009iros}. The LIP itself has been the focus of a large part of the recent literature, in the wake of major works such as~\cite{kajita2003icra, wieber2006humanoids, takenaka2009iros, koolen2012ijrr}. Solutions allowing CoM height variations have therefore been the exception more than the rule. They can be grouped into two categories: pre-planning of CoM height functions, and 2D sagittal capturability.

When CoM height variations $c^z(t)$ are pre-planned~\cite{terada2007humanoids, herdt2012humanoids, hopkins2014humanoids, kamioka2015iros}, the remainder of the system can be controlled in the 2D horizontal plane similarly to the LIP, yet with a time-variant rather than time-invariant equation of motion. Two successful LIP solutions have been generalized following this idea: linear model predictive control~\cite{wieber2006humanoids} was extended into~\cite{herdt2012humanoids}, and the time-invariant divergent component of motion~\cite{englsberger2015tro} was extended into a time-variant counterpart~\cite{hopkins2014humanoids}.\footnote{
    Both~\cite{herdt2012humanoids} and~\cite{hopkins2014humanoids} use polynomial CoM height functions, which makes it easy to satisfy boundary conditions but yields non-integrable dynamics. Terada and Kuniyoshi~\cite{terada2007humanoids} proposed a symmetric alternative where the system becomes integrable, yet where enforcing boundary conditions is a nonlinear root finding problem.
} Interestingly, in~\cite{hopkins2014humanoids} Hopkins \emph{et al.} use the Riccati equation~\eqref{eq:riccati-omega} to compute $\omega(t)$ from $c^z(t)$, while in the present study we compute $\bfc(t)$ from $\omega(t)$. More generally, our strategy can be seen as mapping the whole problem onto the damping $\omega$ and solving for $\omega(t)$, while the underlying strategy behind those other approaches is to fix $\omega$ and map the remainder of the problem onto $\bfc^{xy}$. Another noteworthy example of the latter can be found in the linearized MPC proposed by Brasseur \emph{et al.}~\cite{brasseur2015humanoids}, where $\omega$ variations are this time abstracted using polyhedral bounds rather than a pre-planned height trajectory.

In this regard, our present study is more akin to works on capturability proposed for the 2D nonlinear inverted pendulum~\cite{pratt2007icra, ramos2015humanoids, koolen2016humanoids}. All of them share a design choice dating back to the seminal work of Pratt and Drakunov~\cite{pratt2007icra}: they interpolate CoM trajectories in a 2D vertical plane with a fixed center of pressure (CoP). The key result of~\cite{pratt2007icra} is the conservation of the ``orbital energy'' of a CoM path, a variational principle that we can now interpret as a two-dimensional formulation of the boundedness condition. This principle was later translated into a predictive controller in an equally inspirational study by Koolen \emph{et al.}~\cite{koolen2016humanoids}. Ramos and Hauser~\cite{ramos2015humanoids} also noticed that the capture point, interpreted as \emph{point where to step}, was a function of the CoM path. They proposed a single-shooting method to compute what we would now call 2D capture trajectories.

\begin{figure}[t]
    \centering
    \includegraphics[width=\columnwidth]{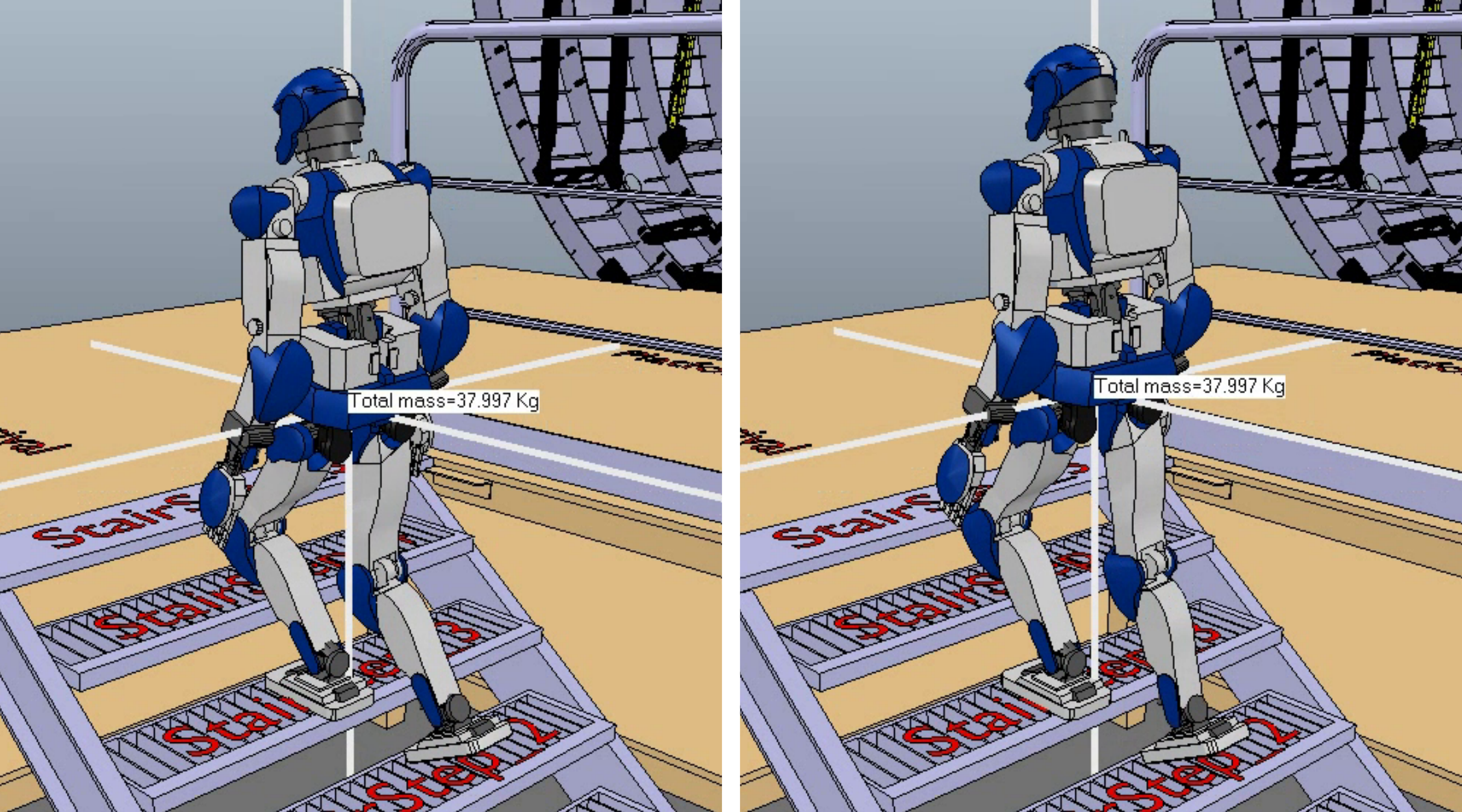}
    \caption{
        \textbf{Dynamic simulations in V-REP:} snapshots at the toe liftoff time instant. Simulations use the Newton dynamics engine with a time step of 5~ms and one pass per frame. \textit{Left:}~LIP-based pattern generation. \textit{Right:}~VHIP-based pattern generation. The former causes HRP-4 to bend its knees more, raising the risk of shank collision with the staircase. 
    }
    \label{fig:vrep}
\end{figure}

All of these works hinted at key features of 3D capture trajectories, but applied only to two-dimensional CoM motions in vertical planes. The key to lift this restriction is the 3D boundedness condition, which was first formulated in the case of the LIP by Lanari \emph{et al.}~\cite{lanari2014humanoids} and applied to model predictive control of the LIP in~\cite{scianca2016humanoids}. This condition can be more generally applied to different asymptotic behaviors, including but not restricted to stopping. For instance, infinite stepping is another option~\cite{lanari2015humanoids}. The exploration of these more general asymptotic behaviors is an open question.

Another important choice of the present study is to focus on zero- and one-step capture trajectories. Walking controllers based on one-step capturability have been proposed for both even~\cite{sugihara2009tro, khadiv2016humanoids, sugihara2017iros} and uneven terrains~\cite{caron2016humanoids, caron2017iros}. The latter follow a single line of work leading to the present study:~\cite{caron2016humanoids} finds rough-terrain (even multi-contact) solutions but tends to produce conservatively slow trajectories;~\cite{caron2017iros} discovers dynamic walking patterns, but suffers from numerical instabilities when used in a closed control loop. In our understanding, these instabilities are due to the \emph{direct} transcription of centroidal dynamics, which has proved successful for planning~\cite{dai2014humanoids, carpentier2016icra} but where closed-loop controllers suffer from frequent switches between local optima~\cite{caron2017iros}. The optimization of capture problems provides an alternative transcription for which we do not observe this numerical sensitivity. 

Finally, the last key choice of the present study is the change of variable from $t$ to $s$. This choice is one possible generalization of the seminal idea by Pratt and Drakunov~\cite{pratt2007icra} to make the CoM height a function $c^z(c^x)$ of the 2D CoM abscissa: as it turns out, $c^x(t)$ and $s(t)$ are proportional in their 2D setting~\cite{caron2018icra}, although that is not the case any more in 3D. An alternative 3D generalization is to solve for the remaining lateral motion $c^y(t)$ after the sagittal motion has been computed by the 2D method~\cite{kajita2017humanoids}. Both cases, as noted in~\cite{kajita2017humanoids} and in the present study, bear a close connection with time-optimal path parameterization (see \emph{e.g.}~\cite{pham2018tro} for a survey). Future works may explore this connection, and perhaps bring to light computational complexity results regarding the best case performance one can hope for this kind of problems.

\begin{figure}[t]
    \centering
    \includegraphics[width=\columnwidth]{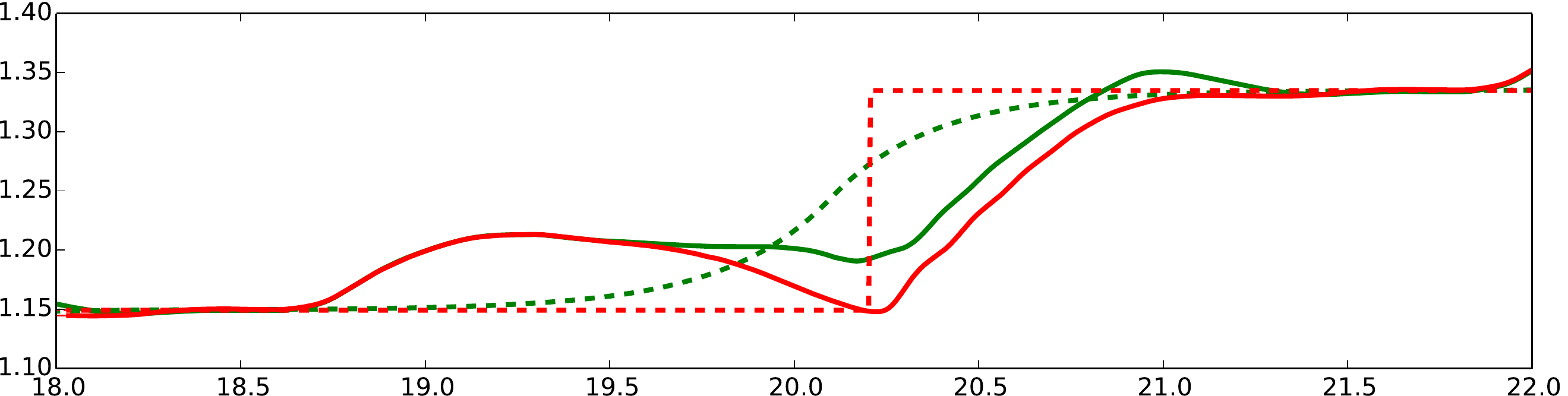}
    \caption{
        \textbf{CoM height variations} for LIP-based (red) and VHIP-based (green) walking patterns on a representative step. Dashed lines: references from walking patterns. Solid lines: estimates by the controller's CoM observer.Tracking cannot be perfect as walking patterns do not take into account swing foot motions (18.5--20.0~s) and maximum leg extension during double support (20.0--20.2~s) that limit the CoM height kinematically.
    }
    \label{fig:com-height}
\end{figure}

\begin{figure}[t]
    \centering
    \includegraphics[width=\columnwidth]{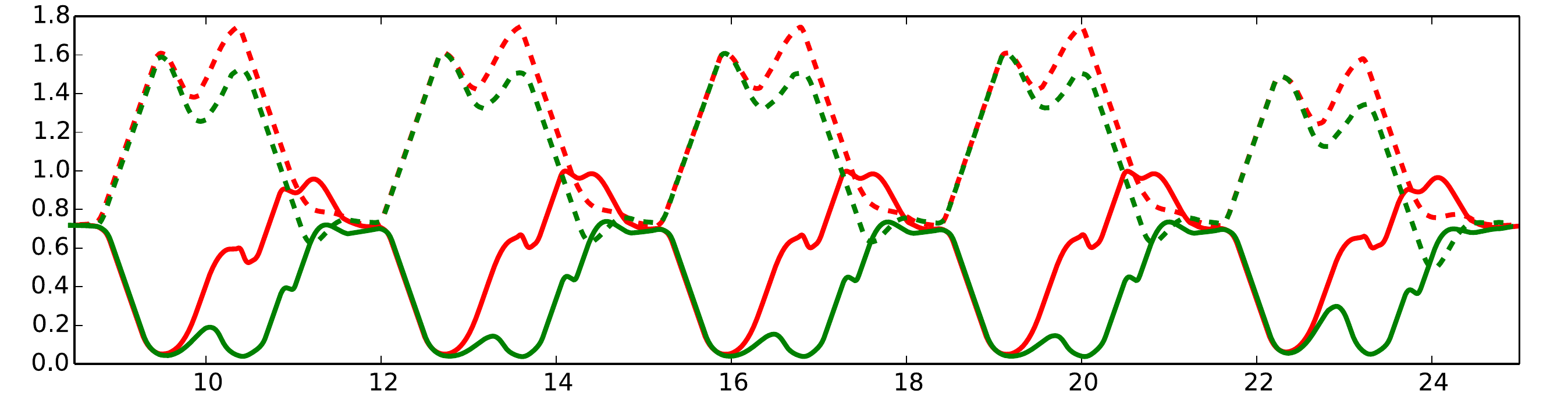}
    \caption{
        \textbf{Knee angle variations} between LIP-based (red) and VHIP-based (green) whole-body control. Solid lines correspond to the right knee and dashed ones to the left knee. The right knee stays extended longer in the VHIP pattern, curtailing a peek in left-knee flexion observed in the LIP pattern.
    }
    \label{fig:knees}
\end{figure}

\begin{figure}[t]
    \centering
    \includegraphics[width=\columnwidth]{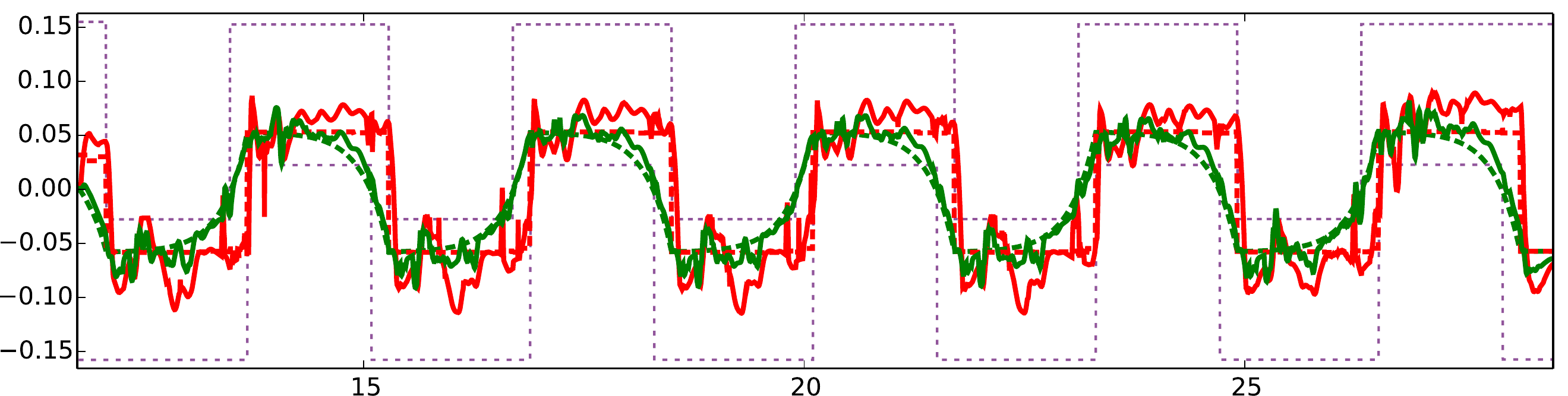}
    \caption{
        \textbf{Lateral DCM-ZMP tracking for the VHIP} over the complete stair climbing motion. Green: DCM, red: ZMP, dashed: reference, solid: observed. The DCM here is the position $\bfc(t) + \bfcd(t) / \omega(t)$. References look like a regular capture-point pattern generator as the variations of $\omega(t)$ required to lift the robot up by 18.5~cm are too small to affect the robot's lateral dynamics. The same behavior is observed on sagittal dynamics.
    }
    \label{fig:lateral-vhip}
\end{figure}

\begin{figure}[t]
    \centering
    \includegraphics[width=\columnwidth]{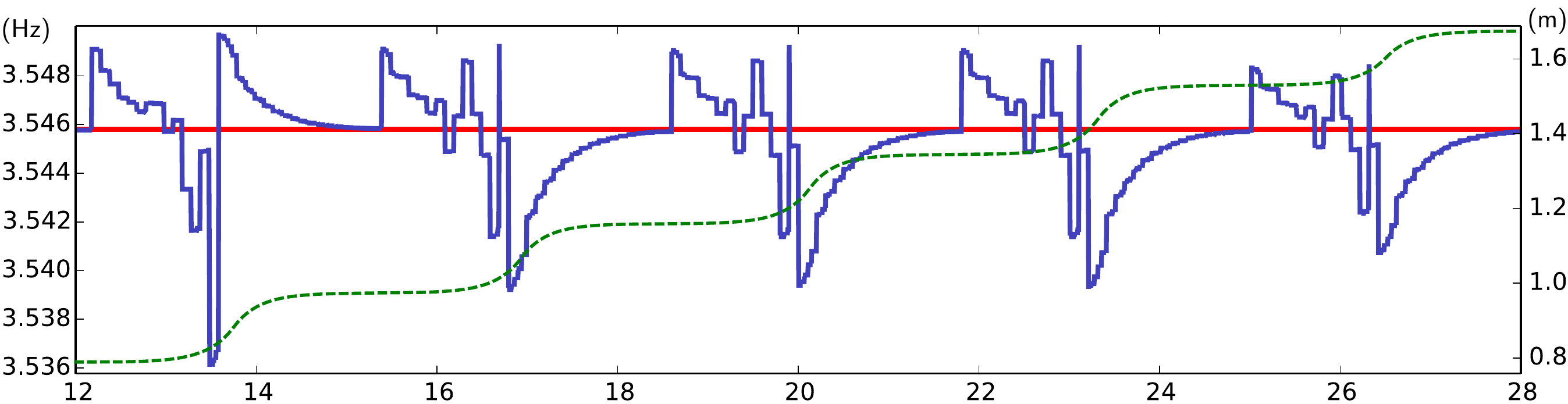}
    \caption{
        \textbf{Time-varying frequency $\omega(t)$} of the VHIP (blue) compared to its LIP counterpart (red). The VHIP pattern CoM height (Figure~\ref{fig:com-height}) is shown for reference in green. The frequency raises when the DCM accelerates upward and falls when the DCM accelerates downward.
    }
    \label{fig:omega-variations}
\end{figure}

The ability to solve capture problems in tens of microseconds opens new perspectives for motion planning and control. For instance, planners can use this tool for fast evaluation of contact reachability, while stepping stabilizers can evaluate several contact candidates in parallel in reaction to \emph{e.g.} external pushes. These extensions are open to future works.

\section*{Acknowledgment}

The author would like to thank Boris van Hofslot, Alain Micaelli and Vincent Thomasset for their feedback on preliminary versions of this paper.
Our thanks also go to Pierre Gergondet for developing and helping us with the mc\_rtc framework.
This research was supported in part by H2020 EU project COMANOID \url{http://www.comanoid.eu/}, RIA No 645097, and the CNRS--AIST--AIRBUS Group Joint Research Program.

\bibliographystyle{IEEEtran}
\bibliography{refs}

\appendices

\section{Mathematical complement}
\label{app:math-background}

In this Appendix, we provide formal proofs for claims made in Section~\ref{sec:analysis} and verify that all quantities are soundly defined. 

\subsection{Restriction to convergent input functions}
\label{app:math-convergent-inputs}

\begin{property}
    \label{prop:capture-traj}
    For every pair of states $\bfxinit$ and $\bfxf=(\bfcf,\bm{0})$, $\traj_{\bfxinit,\bfxf}$ is non-empty if and only if $\traj^c_{\bfxinit,\bfxf}$ is non-empty.
\end{property}
In other words, if there exists a capture input $\lambda(t),\bfr(t)$ steering an initial state $\bfxinit$ to a static equilibrium $\bfxf$, then there exists another input $\lambda^c(t),\bfr^c(t)$ accomplishing the same while also converging.

\begin{proof}
As $\traj^c_{\bfxinit,\bfxf} \subset \traj_{\bfxinit,\bfxf}$, it is enough to prove that $\traj^c_{\bfxinit,\bfxf}$ is non-empty as soon as $\traj_{\bfxinit,\bfxf}$ is non-empty, \emph{i.e.} that given a input function $\lambda(t),\bfr(t) \in \traj_{\bfxinit,\bfxf}$ we can find another input $\lambda^c(t), \bfr^c(t)$ steering to the same state while converging. To this end, consider the following state-dependent inputs:
\begin{align}
    \sqrt{\bar{\lambda}(\bfx)} & = 2\frac{\sqrt{g(5 h - \hf) + \dot{h}^2} - \hdi}{5 \hi - \hf} \\
    \bar{\bfr}(\bfx) & = \bfc + \frac{\bfc - \bfcf}{4} + \frac{\bfcd}{\sqrt{\bar{\lambda}(\bfx)}} +\frac{\bfg}{\bar{\lambda}(\bfx)}
\end{align}
This definition is chosen so that $\bar{\lambda}(\bfx)$ is the solution of:
\begin{equation}
    \left[\bfc - \bfo + \frac{\bfc - \bfcf}{4}\right] \cdot \bfn + \frac{(\bfcd \cdot \bfn)}{\sqrt{X}} + \frac{(\bfg \cdot \bfn)}{X} = 0
\end{equation}
As a consequence, $(\bar{\bfr}(\bfxinit) - \bfo) \cdot \bfn = 0$ and the state-dependent CoP belongs to the contact area. Moreover, $\bar{\lambda}$ and $\bar{\bfr}$ are continuous functions of $\bfxinit$ in a neighbourhood of $\bfxf$, and $\bar{\lambda}(\bfxf) = \lambdaf(\bfcf)$ and $\bar{\bfr}(\bfxf) = \bfrf(\bfcf)$. Hence, as long as $\bfx(t)$ is close enough to $\bfxf$, both $\bar{\lambda}(\bfx(t))$ and $\bar{\bfr}(\bfx(t))$ are feasible.

Injecting those inputs into \eqref{eq:vhip} yields the nonlinear differential equation:
\begin{equation}
    \label{eqn:ipm-inject}
    \bfcdd(t) = -\frac{\bar{\lambda}(\bfx(t))}{4}(\bfc(t)-\bfcf) - \sqrt{\bar{\lambda}} \bfcd(t)
\end{equation}
It is immediate that $\bfxf$ is an equilibrium for this dynamics. The linearized system around this equilibrium is:
\begin{equation}
  \bfcdd^\ell(t) = -\frac{\lambdaf}{4} (\bfc^\ell(t) - \bfcf) - \sqrt{\lambdaf} \bfcd^\ell(t)
\end{equation}
for which the equilibrium is stable. Therefore the equilibrium $\bfxf$ of \eqref{eqn:ipm-inject} is locally stable: if $\bfx(0)$ is close enough to $\bfxf$, then $\bfx(t)$ remains close to $\bfxf$ and converges toward this limit.

We now consider a generic input function $\lambda(t), \bfr(t) \in \traj_{\bfxinit,\bfxf}$. By definition, the solution of \eqref{eq:vhip} converges to $\bfxf$ as $t \to \infty$. Then, there exists some time $T$ such that $\bfx(T)$ is close enough to $\bfxf$ so that, starting from this position, the state-dependent control remains feasible and converges to $\bfxf$. We conclude by noting that the input function that switches at time $T$ from $\lambda, \bfr$ to $\bar{\lambda}, \bar{\bfr}$ belongs to $\traj_{\bfxinit,\bfxf}^c$.
\end{proof}

\subsection{Solutions to the Riccati equations}
\label{app:math-riccati-solutions}

Let us verify the existence of damping solutions and exhibit some of their properties that will prove useful to characterize those that don't diverge. 

\begin{property}
    \label{prop:omega-bounds}
    Assume that we are given $\lambda$ such that $\lambda(t) \in [\lambdamin,\lambdamax]$ at all times $t$. Then, there exists a unique $\omegai > 0$ such that the solution $\omega$ of \eqref{eq:riccati-omega} with $\omega(0) = \omegai$ is positive and finite at all times. Moreover, this solution is such that:
    \begin{equation}
        \label{eq:omega-bounds}
        \forall t > 0, \quad \omega(t) \in [\sqrt{\lambdamin},\sqrt{\lambdamax}]
    \end{equation}
\end{property}
In other words, there is a one-to-one mapping between the stiffness function $\lambda(t)$ and its non-diverging filtered damping $\omega(t)$. 

\begin{proof}
Consider first the case of a constant input $\lambda$. One can note that the differential equation $\dot{y} = y^2 - \lambda$ has two equilibrium points: one stable $-\sqrt{\lambda}$ and one unstable $\sqrt{\lambda}$. More precisely, given $y_0 \in \mathbb{R}$, the only solution of this equation satisfying $y(0) =y_0$ is given by
\begin{equation}
  y(t) = \begin{cases}
    \frac{\sqrt{\lambda}}{\tanh(\sqrt{\lambda} (T-t))} & \text{ if } |y_0| > \sqrt{\lambda} \\
    \sqrt{\lambda}\tanh(\sqrt{\lambda}(T-t)) & \text{ if } |y_0| < \sqrt{\lambda} \\
    y_0 & \text{ if } |y_0|=\sqrt{\lambda}
  \end{cases}
\end{equation}
The initial condition $y_0$ settles the behavior of the solution at all times. Define the time $T = \frac{1}{2\sqrt{\lambda}} \log \left|\frac{y_0 - \sqrt{\lambda}}{y_0+\sqrt{\lambda}}\right|$, then:
\begin{itemize}
  \item If $0 \leq y_0 < \sqrt{\lambda}$, then $\lim_{t \to \infty} y(t) = -\sqrt{\lambda}$ and the solution $y$ becomes negative after time $T$.
  \item If $y_0 = \sqrt{\lambda}$, then $y(t) = \sqrt{\lambda}$ for all $t > 0$.
  \item If $y_0 > \sqrt{\lambda}$, then $\lim_{t \to T} y(t) = +\infty$: the solution explodes in finite time.
\end{itemize}

Let us move now to the general case where $\lambda(t)$ is time-varying, and denote by $\omega$ a non-negative, non-explosive solution to \eqref{eq:riccati-omega}. If $\omega(t_0)>\sqrt{\lambdamax}$ at some time $t_0 > 0$, then choosing the solution $y$ to $\dot{y} = y^2 - \lambdamax$ with $y(t_0) = \omega(t_0)$, we observe that, as long as $0 \leq y(t)\leq \omega(t)$,
\begin{equation}
    \dot{\omega}(t) - \dot{y}(t) = \omega^2(t) - y^2(t) + \lambdamax - \lambda(t)  \geq 0
\end{equation}
Therefore, $\omega - y$ is nondecreasing, and $y(t) \leq \omega(t)$ holds until the explosion time $T$ of $y$. This shows that $\omega$ explodes in finite time, in contradiction with the hypothesis. Similarly, if $\omega(t_0)< \sqrt{\lambdamin}$ for some $t_0 > 0$, we can upper-bound $\omega(t)$ by the solution to $\dot{y} = y^2 - \lambdamin$ such that $y(t_0) = \omega(t_0)$, which becomes negative in finite time, once again contradicting the hypothesis. The bounds~\eqref{eq:omega-bounds} must therefore hold.

We now prove the uniqueness of the non-negative non-exploding solution of \eqref{eq:riccati-omega}. Suppose that one could find two such solutions, $\omega_1$ and $\omega_2$. As observed above, these functions remain in the interval $[\sqrt{\lambdamin},\sqrt{\lambdamax}]$. Consider without loss of generality that $\omega_1(0)>\omega_2(0)$. Then, as long as $\omega_1(t)>\omega_2(t)$,
\begin{equation}
    \dot{\omega_1} - \dot{\omega_2} = (\omega_1 - \omega_2)(\omega_1 + \omega_2) \geq 2 \sqrt{\lambdamin} (\omega_1 - \omega_2)
\end{equation}
As a consequence, $\omega_1(t) - \omega_2(t) \geq (\omega_1(0)-\omega_2(0)) e^{2 \sqrt{\lambdamin} t}$ at all times, showing that the two functions cannot be bounded at the same time.

To finally prove the existence of the solution, we observe there exists a unique $\omega(0) \in [\sqrt{\lambdamin},\sqrt{\lambdamax}]$ such that for all $t>0$, we have $\overline{y}_t(0) > \omega(0) > \underline{y}_t(0)$ where $\overline{y}_t$ and $\underline{y}_t$ are the solutions of \eqref{eq:riccati-omega} with conditions $\overline{y}_t(t) = \sqrt{\lambdamax}$ and $\underline{y}_{t}(t) = \sqrt{\lambdamin}$. The solution $\omega$ starting from this value $\omega(0)$ remains within bounds by construction (the time it crosses $\sqrt{\lambdamin}$ or $\sqrt{\lambdamax}$ is greater than any finite time $t$).
\end{proof}

Let us now turn to the other damping $\gamma$. While there is a unique solution $\omega$ corresponding to a given $\lambda$, there are many different non-negative finite functions $\gamma$ that satisfy \eqref{eq:riccati-gamma}. As a matter of fact, each choice of $\gamma(0) > 0$ yields an admissible solution:
\begin{property}
    \label{prop:gamma-bounds}
    Assume that we are given $\lambda$ such that $\lambda(t) \in [\lambdamin,\lambdamax]$ at all times $t$. For all $\gamma(0) > 0$, the solution $\gamma$ of \eqref{eq:riccati-gamma} is non-negative and finite at all times. Moreover:
    \begin{equation}
        \sqrt{\lambdamin} \leq \liminf_{t \to \infty} \gamma(t) \leq \limsup_{t \to \infty} \gamma(t) \leq \sqrt{\lambdamax}
    \end{equation}
\end{property}

\begin{proof}
The existence and the uniqueness of the solution on a maximal interval are consequences of the Cauchy-Lipschitz theorem. As in the previous proof, we can compare $\gamma$ with the functions $\underline{y}$ and $\overline{y}$, respectively solutions to $\dot{y} = \lambdamin - y^2$ and $\dot{y} = \lambdamax - y^2$ with $\underline{y}(0) = \overline{y}(0) = \gamma(0)$. Then, $\underline{y}(t) \leq \gamma(t) \leq \overline{y}(t)$ at all times $t$, and the rest of the proof is a consequence of $\lim_{t \to \infty} \underline{y}(t) = \sqrt{\lambdamin}$ and $\lim_{t \to \infty} \overline{y}(t) = \sqrt{\lambdamax}$.
\end{proof}

An interesting consequence of these two properties is the following asymptotic behavior:
\begin{corollary}
\label{cor:cv}
If $\lim_{t \to \infty} \lambda(t) = \lambdaf$, then
\begin{equation}
  \lim_{t \to \infty} \omega(t) = \lim_{t \to \infty} \gamma(t) = \sqrt{\lambdaf}.
\end{equation}
\end{corollary}

\begin{proof}
We only consider the case of $\omega$, the proof for $\gamma$ following the same derivation. By definition of the limit, for any $\epsilon>0$, there exists $t_0>0$ large enough so that $\forall t > t_0, |\lambda(t) - \lambdaf| < \epsilon$. Next, remark that the time-shifted function $\widetilde{\omega}(t) \defeq \omega(t+t_0)$ is a solution of the equation $\dot{\widetilde{\omega}}(t) = \widetilde{\omega}(t)^2 - \widetilde{\lambda}$, where $\forall t>0, \widetilde{\lambda}(t) = \lambda(t+t_0) \in [\lambdaf - \epsilon, \lambdaf + \epsilon]$. Property \ref{prop:omega-bounds} then shows that $\omega(t + t_0) \in [\sqrt{\lambdaf-\epsilon},\sqrt{\lambdaf + \epsilon}]$ for all $t > 0$.
As a consequence, for any $\epsilon > 0$:
\begin{equation}
    \sqrt{\lambdaf - \epsilon} \leq \liminf_{t \to \infty} \omega(t) \leq \limsup_{t \to \infty} \omega(t) \leq \sqrt{\lambdaf + \epsilon}
\end{equation}
Taking the limit as $\epsilon \to 0$, we conclude that the limit of $\omega$ as $t \to \infty$ exists and is equal to $\sqrt{\lambdaf}$.
\end{proof}

We conclude from the above properties that the solutions $\bfzeta$ and $\bfxi$ to Equation~\eqref{eq:decoupledsystem} are well-defined.

\subsection{Convergent component of motion}
\label{app:ccm}

The component $\bfzeta$ corresponding to the damping $\gamma$ is subject to the differential equation:
\begin{equation}
  \label{eq:zeta}
    \bfzetad = - \gamma \bfzeta + (\lambda \bfr - \bfg)
\end{equation}
The general solution to this equation is given by:
\begin{equation}
    \label{eq:zeta(t)}
    \bfzeta(t) = \left(\bfzeta(0) + \int_0^t e^{\Gamma(\tau)} (\lambda(\tau) \bfr(\tau) - \bfg) \dd{\tau}\right) e^{-\Gamma(t)}
\end{equation}
where $\Gamma$ is the antiderivative of $\gamma$ such that $\Gamma(0)=0$, \emph{i.e.}~$\Gamma(t) = \int_0^t \gamma(t) \dd{t}$. It satisfies the following two identities:
\begin{align}
    \frac{\dd{e^{\Gamma}}}{\dd{t}} & = \gamma e^{\Gamma} &
    \frac{{\rm d}^2{e^{\Gamma}}}{\dd{t}^2} & = (\gammad + \gamma^2) e^{\Gamma} = \lambda e^{\Gamma}
\end{align}
The asymptotic behavior of $\bfzeta$ is tied to that of the two inputs $\lambda$ and $\bfr$ of the inverted pendulum:
\begin{property}
    \label{prop:zeta-conv}
    Consider an input function $\lambda(t), \bfr(t)$ such that $\lim_{t \to \infty} \lambda(t) = \lambdaf$ and $\lim_{t\to \infty} \bfr(t) = \bfrf$, and let $\gamma$ denote any solution to \eqref{eq:riccati-gamma}. Then, the solution $\bfzeta$ of \eqref{eq:zeta} satisfies:
    \begin{equation}
      \label{eq:zeta-limit}
      \lim_{t \to \infty} \bfzeta(t) = \sqrt{\lambdaf} \left(\bfrf - \frac{\bfg}{\lambdaf}\right) = \sqrt{\lambdaf} \bfcf
    \end{equation}
\end{property}

\begin{proof}
By Corollary \ref{cor:cv}, $\lim_{t \to \infty} \gamma(t) = \sqrt{\lambdaf}$, therefore its antiderivative $\Gamma(t)$ satisfies $\lim_{t \to \infty} \Gamma(t)/t = \sqrt{\lambdaf}$ as well. In particular, $\Gamma(t)$ diverges to $\infty$, so that:
\begin{equation}
    \bfzeta(t) \underset{t \to \infty}{\sim} e^{-\Gamma(t)} \int_0^t e^{\Gamma(\tau)} (\lambda(\tau) \bfr(\tau) - \bfg) \dd{\tau}
\end{equation}
where the notation $f \sim g$ means that the ratio $f/g$ goes to 1 as $t \to \infty$. Applying l'H\^opital's rule, we conclude that:
\begin{align}
     \lim_{t \to \infty} \bfzeta(t) & = \lim_{t \to \infty} \frac{\int_0^t e^{\Gamma(\tau)} (\lambda(\tau) \bfr(\tau) - \bfg) \dd{\tau}}{e^{\Gamma(t)}}
     \\ 
     & = \lim_{t \to \infty} \frac{e^{\Gamma(t)}(\lambda(t) \bfr(t) - \bfg)}{\gamma(t)e^{\Gamma(t)}} = \frac{(\lambdaf \bfrf - \bfg)}{\sqrt{\lambdaf}}
     \qedhere
\end{align}
\end{proof}

\subsection{Divergent component of motion}
\label{app:dcm}

The general solution~\eqref{eq:xi(t)} of the divergent component of motion is based on the antiderivative $\Omega(t) = \int_0^t \omega(t) \dd{t}$. Recalling from Property~\ref{prop:omega-bounds} that $\omega \in [\sqrt{\lambdamin}, \sqrt{\lambdamax}]$, we see that $\Omega$ grows at least linearly. Therefore, as long as $\lambda$ and $\bfr$ remain bounded, the integral $\int_0^\infty e^{-\Omega(\tau)} (\bfg - \lambda(\tau) \bfr(\tau)) \dd{\tau}$ is well-defined and finite. Let us now prove Property~\ref{prop:xi-conv}.

\begin{proof}
The proof is very similar to that of Property \ref{prop:zeta-conv}. The solution \eqref{eq:xi(t)} with $\bfxi(0) = \bfxiinit$ from Equation~\eqref{eq:xi-initial} becomes:
\begin{equation}
    \bfxi(t) = e^{\Omega(t)} \int_t^\infty e^{-\Omega(\tau)} (\lambda(\tau)\bfr(\tau) - \bfg) \dd{\tau}
\end{equation}
Applying l'H\^opital's rule, we conclude by Corollary~\ref{cor:cv} that:
\begin{align}
    \lim_{t \to \infty} \bfxi(t) & = \lim_{t \to \infty} \frac{e^{-\Omega(t)} (\lambda(t)\bfr(t) - \bfg)}{\omega(t) e^{-\Omega(t)}} = \frac{\lambdaf \bfrf - \bfg}{\sqrt{\lambdaf}}
\end{align}
That is to say, similarly to the convergent component of motion, the solution $\bfxi$ of \eqref{eq:xi-def} converges to $\sqrt{\lambdaf} \bfcf$.
\end{proof}

\subsection{Proof of Property~\ref{prop:capture-inputs}}
\label{app:capture-inputs}

\begin{proof}[Proof of the $\Rightarrow$ implication]
    Let $\lambda(t), \bfr(t)$ denote a capture input from $\traj^c_{\bfxinit,\bfxf}$, with $\bfx(t)$ the smooth trajectory resulting from this input via the equation of motion~\eqref{eq:vhip}. Its boundary values $\bfx(0) = \bfxinit$ and $\bfx(\infty) = \bfxf$ being bounded, this trajectory must be bounded as well. As $\bfzeta(t)$ is always bounded by Property \ref{prop:zeta-conv}, this in turns implies that its the divergent component $\bfxi(t)$ is bounded, and must therefore satisfy Equation~\eqref{eq:boundedness} by Property~\ref{prop:xi-conv}. Next, let us denote by $\lambdaf, \bfrf$ the limits of $\lambda(t),\bfr(t)$ as time goes to infinity. Using Properties \ref{prop:zeta-conv} and \ref{prop:xi-conv}, the two components converge to:
    \begin{equation}
      \label{eqn:cvxizeta}
      \lim_{t \to \infty} \bfzeta(t) = 
      \lim_{t \to \infty} \bfxi(t) = 
      \sqrt{\lambdaf}\left(\bfrf - \frac{\bfg}{\lambdaf}\right)
    \end{equation}
    Recalling from Corollary~\ref{cor:cv} that $\gamma$ and $\omega$ converge to $\sqrt{\lambdaf}$, we can take the limit in the mapping \eqref{eq:S-Sinv}--\eqref{eq:def-xi}:
    \begin{equation}
        \lim_{t \to \infty} \bfc(t) = \lim_{t \to \infty} \frac{\bfzeta(t) + \bfxi(t)}{\gamma(t) + \omega(t)} = \bfrf - \frac{\bfg}{\lambdaf} = \bfcf
    \end{equation}
    Therefore, $\lambdaf = \lambdaf(\bfcf)$ and $\bfrf = \bfrf(\bfcf)$.
\end{proof}

\begin{proof}[Proof of the $\Leftarrow$ implication]
   Reciprocally, assuming (\emph{i})--(\emph{iii}), Equation~\eqref{eqn:cvxizeta} holds again by Properties \ref{prop:zeta-conv}--\ref{prop:xi-conv} and Corollary~\ref{cor:cv}. Furthermore,
    \begin{align}
        \lim_{t \to \infty} \bfcd(t) & = \lim_{t \to \infty} \frac{-\omega(t) \bfzeta(t) + \gamma(t) \bfxi(t)}{\gamma(t) + \omega(t)} = \zerovec
    \end{align}
    Thus, the pendulum driven by $\lambda(t), \bfr(t)$ converges to the static equilibrium $\bfxf$.
\end{proof}

\subsection{Note on fixed-CoP strategies}
\label{app:fixed-cop}

When the CoP input is stationary, \emph{i.e.}~in a point-foot model, the Gram determinant $G = ((\bfc - \bfr) \times \bfcd) \cdot \bfg$ becomes invariant.

\begin{proof}[Short proof]
Take the cross-product of Equation~\eqref{eq:vhip} with $\bfcdd$. Then, the scalar product of the resulting expression with $\bfg$ yields $((\bfc - \bfr) \times \bfcdd) \cdot \bfg = 0$. Conclude by noting that this formula is the time derivative of $((\bfc - \bfr) \times \bfcd) \cdot \bfg$.
\end{proof}

There are two possible outcomes: either $G=0$, which means the three vectors are coplanar and the robot may stabilize using a 2D strategy~\cite{pratt2007icra, ramos2015humanoids, koolen2016humanoids}; or $G \neq 0$ and it is impossible to bring the system to an equilibrium where $\bfcd = 0$. This shows simultaneously two properties: first, that sagittal 2D balance control is the most general solution for point-foot models, and second, that these models have a very limited ability to balance, as they need to re-step at the slightest lateral change in linear momentum. The ability of flat-footed bipeds to absorb these perturbations (to some extent) without stepping comes from continuous CoP variations.

\section{Numerical optimization complement}
\label{app:optim-background}

In this Appendix, we first recall terminology and state-of-the-art algorithms for numerical optimization. We essentially rewrite treatment from~\cite{nocedal:book:2006} for double-sided inequality constraints. We then detail all tailored operations mentioned in Section~\ref{sec:optim}.

\subsection{Definitions and notations}
Consider the optimization problem:
\begin{subequations}
    \label{eq:generic-problem}  
    \begin{align}
        \minimize_{\bfx \in \bbR^n} &\ f(\bfx)\\
	    \subjto &\ \bfl \leq \bfh(\bfx) \leq \bfu
    \end{align}
\end{subequations}
where $f$ and $\bfh$ are smooth functions, $f$ being $1$-dimensional and $\bfh$ $m$-dimensional. Lower and upper bound constraints are represented by vectors $\bfl, \bfu \in \bbR^m$, with equality constraints specified by taking $l_j = u_j$. A point $\bfx$ is \emph{feasible} if it satisfies all constraints. For a given $\bfx$, we say the $\th{j}$ constraint is \emph{active} at its lower (resp. upper) bound when $h_j(\bfx) = l_j$ (resp. $h_j(\bfx) = u_j$). We denote by:
\begin{align}
    \calE & \defeq \left\{j \in \left[1,m\right], l_j=u_j\right\} \\
    \calA(\bfx)^- & \defeq \left\{j \notin \calE, h_j(\bfx) = l_j\right\} \\
    \calA(\bfx)^+ & \defeq \left\{j \notin \calE, h_j(\bfx) = u_j\right\}
\end{align}
These three sets are disjoint. For a set of indexes $\calS$ and a matrix $\bfM$, we define $\bfM_{\calS}$ the matrix made of the rows of $\bfM$ whose indexes are in $\calS$ (this notation also applies to vectors).

The \emph{Lagrangian} of the problem is defined as $\calL(\bfx, \bflambda^-, \bflambda^+) \defeq f(\bfx) + \bflambda^{-T} (\bfh(\bfx) - \bfl) + \bflambda^{+T} (\bfh(\bfx) - \bfu)$ where $\bflambda^-, \bflambda^+ \in \bbR^m$ are the \emph{Lagrange multipliers}, $\bfnabla_{\bfx} g$ and $\bfnabla^2_{\bfx\bfx} g$ are respectively the gradient and Hessian of a function $g$ with respect to $\bfx$. 
We note $\bflambda \defeq \bflambda^- + \bflambda^+$. We can work with it instead of $\bflambda^-$ and  $\bflambda^+$ (see~\cite[\S 4.3.5]{brossette:phd:2016}).

The \emph{Karush--Kuhn--Tucker} (KKT) conditions give necessary conditions on $\bfx$ and $\bflambda$ for $\bfx$ to be a minimizer of Problem~(\ref{eq:generic-problem}) (see~\cite[chap. 12]{nocedal:book:2006}). They are often used as termination conditions in solvers.

\subsection{Active-set method for Quadratic Programming}

When the objective $f$ is quadratic, $f(\bfx) = \frac{1}{2}\bfx^T \bfQ \bfx + \bfq^T \bfx$, with $\bfQ$ symmetric positive semidefinite and $\bfh$ linear, $\bfh(\bfx) = \bfC \bfx$ for some matrix $\bfC$, Problem~\eqref{eq:generic-problem} is a (convex) Quadratic Program with Inequality constraints (QPI). One of the main approaches to solve it is the \emph{active-set} method. This method iteratively discovers the set of constraints active at the solution\footnote{
    We ignore here for the sake of simplicity a subtlety arising when active constraints are linearly dependent.
} by solving at each iteration $k$ the following Quadratic Program with only Equality constraints (QPE):
\begin{subequations}
	\label{eq:QP-subproblem}
    \begin{align}
        \minimize_{\bfp \in \bbR^n} &\ \frac{1}{2} \bfp^T \bfQ \bfp + (\bfQ \bfx_k + \bfq)^T \bfp \\
        \subjto &\ \bfC_{\calW_k} \bfp = 0
    \end{align}
\end{subequations}
where $\calW_k$ is a set of indexes. The solution $\bfp^*$ to this QPE is used to determine the next iterate $\bfx_{k+1}$.

Unlike QPIs, QPEs admit analytical solutions as their KKT conditions reduce to a linear system. For a given $\bfx_k$, we can retrieve $\bflambda$ with:
\begin{align}
	\lambda_{\calA_k} & = -\bfC_{\calA_k}^{\dagger T} \bfnabla_{\bfx}^T f(\bfx_k),
    & \lambda_i & =0, \forall i\notin \calA_k \label{eq:lambda}
\end{align}
where $\square^{\dagger}$ denotes the Moore-Penrose pseudo-inverse and $\calA_k = \calA(\bfx_k)^- \cup \calA(\bfx_k)^+ \cup \calE$ is the active set at $\bfx_k$, \emph{i.e.} the set of constraints that are active at this point.

The active-set method for convex QPIs is given in Algorithm~\ref{alg:active-set}. See~\cite[chap. 16]{nocedal:book:2006} for more details on this method. 

\begin{algorithm}
	\caption{Active-set algorithm for convex QPI}
    \begin{algorithmic}
		\STATE Given a feasible point $\bfx_0$
		\STATE Let $\calW_0^- = \calA(\bfx_0)^-$, $\calW_0^+ = \calA(\bfx_0)^+$
		\FOR{$k = 0,1,2, \ldots$}
		  \STATE Compute $\bfp$ from \eqref{eq:QP-subproblem} with $\calW_k = \calW_k^- \cup \calW_k^+ \cup \calE$
			\IF{$\bfp = 0$}
				\STATE Compute $\bflambda$ using Equation~\eqref{eq:lambda}
				\IF{$\bfx$ and $\bflambda$ verify the KKT conditions}
					\STATE \textbf{return} the solution $\bfxf = \bfx_k$
				\ELSE
					\STATE Choose $j$ such that $\lambda_j$ violates the KKT conditions
					\STATE $\bfx_{k+1} = \bfx_k$, $\calW_{k+1}^- = \calW_k^- \backslash \left\{j\right\}$, $\calW_{k+1}^+ = \calW_k^+ \backslash \left\{j\right\}$
				\ENDIF
			\ELSE
				\STATE Find the largest $\alpha\leq 1$ such that $\bfx_k + \alpha \bfp$ is feasible.
				\STATE $\bfx_{k+1} = \bfx_k + \alpha \bfp$
				\IF{some constraints have been activated doing so}
					\STATE Let $j$ be the index of one of them
					\STATE Obtain $\calW_{k+1}^-$ and $\calW_{k+1}^+$ from $\calW_{k}^-$ and $\calW_{k}^+$ by adding $j$ to the appropriate set
				\ELSE
					\STATE $\calW_{k+1}^- = \calW_{k}^-$, $\calW_{k+1}^+ = \calW_{k}^+$. 
				\ENDIF
			\ENDIF
		\ENDFOR
	\end{algorithmic}
	\label{alg:active-set}
\end{algorithm}

\subsection{Sequential Quadratic Programming}

Sequential quadratic programming (SQP) is an iterative optimization technique for solving general constrained problems such as~\eqref{eq:generic-problem}. At each iteration $k$, a QP approximation of~\eqref{eq:generic-problem} is formed and solved:
\begin{subequations}
	\label{eq:sqp-subproblem}
	\begin{align}
		\minimize_{\bfp \in \bbR^n} & \ f(\bfx_k) + \bfnabla_{\bfx} f(\bfx_k)^T \bfp + \frac{1}{2} \bfp^T \bfB_k \bfp\\
		\subjto & \ \bfl - \bfh(\bfx_k) \leq \bfnabla_{\bfx} \bfh(\bfx_k)^T \bfp \leq \bfu - \bfh(\bfx_k)
	\end{align}
\end{subequations}
where $\bfB_k$ can be $\bfnabla^2_{\bfx\bfx} \calL(\bfx_k, \bflambda_k)$ or, for faster computations, some positive-definite approximation of it.

\begin{algorithm}
	\caption{Line search SQP}
    \begin{algorithmic}
        \STATE Given a stepping parameter $\tau \in (0, 1)$
		\STATE Choose $(\bfx_0, \bflambda_0)$
		\WHILE{the KKT conditions are not satisfied}
		  \STATE Compute $\bfp$ from~\eqref{eq:sqp-subproblem}
			\STATE Let $\bflambda$ be the corresponding multiplier and $\bfp_{\lambda} = \bflambda - \bflambda_k$
			\STATE $\alpha = 1$
			\WHILE{$\alpha \bfp$ does not yield an acceptable step}
				\STATE $\alpha = \tau \alpha$
			\ENDWHILE
			\STATE $\bfx_{k+1} = \bfx_k + \alpha \bfp$, $\bflambda_{k+1} = \bflambda_k + \alpha \bfp_{\lambda}$
		\ENDWHILE
	\end{algorithmic}
	\label{alg:sqp}
\end{algorithm}

The outline of the SQP method is given in Algorithm~\ref{alg:sqp}. There are several criteria for assessing whether a step is acceptable, see~\cite[chap. 18]{nocedal:book:2006} for details.

\subsection{Cost matrix of the unconstrained least squares problem}
\label{app:build-T}

Consider the active set $\calW$ for a given SQP iteration. Starting at the first constraint, count the number $a_0$ of consecutive active constraints (possibly $0$ if the first constraint is not active), then $j_1$ the number of following consecutive inactive constraints, $a_1$ the number of following active constraints,~\emph{etc.} The set $\calW$ is then fully described by the sequence $(a_0, j_1, a_1, j_2, a_2, \ldots, j_p, a_p)$, where only $a_0$ and $a_p$ can be $0$. Note that $\sum_k a_k + \sum_k j_k = n+1$, and let us define $n_{\calW} \defeq \sum_k a_k$.
For example, if $\calW = \left\{1, 2, 6, 9, 10, 11, 13, 14\right\}$ for $n=15$ optimization variables, we get the sequence $(2,3,1,2,3,1,2,2,0)$ and $n_{\calW} = 8$.
The constraint matrix $\bfC_\calW$ is then the $n_{\calW}\times n$ matrix:
\begin{equation}
\bfC_\calW = 
  \BIN \bfC_0 &     &     &     &     &        &     &    \\
           & \hspace{-5pt}\zeromat{a_1}{j_1-1} & \hspace{-3pt}\bfC_1 &     &     &        &     &    \\
           &     &     & \hspace{-5pt}\zeromat{a_2}{j_2-1} & \hspace{-3pt}\bfC_2 &        &     &    \\
           &     &     &     &     & \ddots &     &    \\
           &     &     &     &     &        & \hspace{-12pt}\zeromat{a_p}{j_p-1} & \hspace{-3pt}\bfC_p\BOUT
\end{equation}
where $\zeromat{m}{q}$ is the $m \times q$ zero matrix, while $\bfC_0$, $\bfC_k$ ($k\in[1, p-1]$) and $\bfC_p$ are respectively $a_0 \times a_0$, $a_k \times (a_k+1)$ and $a_p \times a_p$ matrices ($\bfC_0$ and $\bfC_p$ can be empty) of the form:
\begin{equation}
\small
  \bfC_0 \hspace{-2pt}=\hspace{-3pt} 
	      \BIN 1 & \hspace{-5pt}       & \hspace{-5pt}       & \hspace{-5pt}  \\
            -1 & \hspace{-5pt}   1   & \hspace{-5pt}       & \hspace{-5pt}  \\
               & \hspace{-5pt}\ddots & \hspace{-5pt}\ddots & \hspace{-5pt}  \\
               & \hspace{-5pt}       & \hspace{-7pt}  -1   & \hspace{-6pt}1 \BOUT\hspace{-3pt}, \
  \bfC_k \hspace{-2pt}=\hspace{-3pt} 
	      \BIN-1 & \hspace{-5pt}   1   & \hspace{-5pt}       & \hspace{-5pt}  \\
               & \hspace{-5pt}\ddots & \hspace{-5pt}\ddots & \hspace{-5pt}  \\
               & \hspace{-5pt}       & \hspace{-7pt}  -1   & \hspace{-6pt}1 \BOUT\hspace{-3pt}, \
  \bfC_p \hspace{-2pt}=\hspace{-3pt} 
	      \BIN-1 & \hspace{-5pt}   1    & \hspace{-5pt}       & \hspace{-5pt}  \\
               & \hspace{-5pt}\ddots  & \hspace{-5pt}\ddots & \hspace{-5pt}  \\
               & \hspace{-5pt}        & \hspace{-7pt}  -1   & \hspace{-6pt}1 \\
               & \hspace{-5pt}        & \hspace{-7pt}       & \hspace{-6pt}1 \BOUT
\end{equation}

Denoting by $\bm{1}_a$ the vector of size $a$ filled with ones, the nullspace projection matrix for the active set $\calW$ is:
\begin{equation}
  \bfN_{\mathcal{W}} = \BIN 
    \zeromat{a_0}{i_1-1} & & & & & \\
    \bfI_{i_1-1} & & & & & \\
    & \hspace{-5pt}\bm{1}_{a_1+1} & & & & \\
    && \hspace{-5pt}\bfI_{i_2-1} & & & \\
    &&& \hspace{-5pt}\bm{1}_{a_2+1} & & \\
    &&&&\ddots & \\
    &&&&& \hspace{-7pt}\bfI_{i_p-1} \\
    &&&&& \hspace{-7pt}\zeromat{a_p}{i_p-1}
  \BOUT
\end{equation}
Noting that $\bfC_k \bm{1}_{a_k+1} = \zeromat{a_k}{1}$, we can directly verify that $\bfC_{\mathcal{W}} \bfN_{\mathcal{W}} = \zerovec$. 
The matrix $\bfN_{\mathcal{W}}$ is $n$ by $n-n_{\mathcal{W}}$ and full column rank. 
It is thus a basis of the nullspace of $\bfC_{\mathcal{W}}$. 

Computing the product $\bfM \bfN_\calW$ for a given matrix $\bfM$ does not actually require to perform any multiplication: multiplying by $\bm{1}$ amounts simply to the summation of columns of $\bfM$. Likewise, $\bfN_\calW \bfz$ just requires to copy the elements of $\bfz$. 
It is thus not necessary to form $\bfN_\calW$, and $\bfT$ can be obtained by $\sum a_k = n_{\calW}$ vector additions. Taking into account the tridiagonal structure of $\bfJ$, this can be done in $O(n)$.

\subsection{Computation of Lagrange multipliers}
\label{app:lagrange-multipliers}

The computation of the Lagrange multipliers, needed to check KKT conditions, relies on the pseudoinverse of $\bfC_{\mathcal{W}}$ (see \emph{e.g.} Equation~(\ref{eq:lambda})). Due to its block structure, expressing the latter is done by finding the pseudoinverse for each $\bfC_k$:
\begin{equation}
  \bfC_{\mathcal{W}}^\dagger = \BIN
    \bfC_0^{-1} & & & & \\
    & \zeromat{i_1-1}{a_1} & & & \\
    & \bfC_1^\dagger& & &\\
    && \zeromat{i_2-1}{a_2} & & & \\
    && \bfC_2^\dagger&  & &\\
    &&& \ddots & \\
    &&&& \zeromat{i_p-1}{a_p} \\
    &&&& \bfC_p^{-1}
  \BOUT
\end{equation}
where $\square^\dagger$ denotes the pseudoinverse. It can be verified that
\begin{align}
  &\bfC_0^{-1} = \BIN 
    1 &   & \\
    \vdots & \ddots & \\
    1 & \ldots & 1 \BOUT, \ 
  \bfC_p^{-1} = \BIN 
    -1 & \ldots & -1 & 1 \\
       & \ddots & \vdots & 1\\
       &        &  -1   & 1 \\
       &&& 1 \BOUT\\
 &\bfC_k^{\dagger} = \frac{1}{a_k+1} \BIN
   -a_k & -(a_k-1) & -(a_k-2) & \ldots & -1 \\
     1  & -(a_k-1) & -(a_k-2) & \ldots & -1 \\
     1  &    2     & -(a_k-2) & \ldots & -1 \\
     1  &    2     &    3     & \ddots & \vdots \\
        & \vdots   &          & \ddots & -1 \\
     1  &    2     &    3     & \ldots & a_k
 \BOUT
\end{align}
The computation of Lagrange multipliers can thus be done in $O(n^2)$ without forming the pseudoinverse explicitly.

\subsection{QR decomposition of the cost matrix \texorpdfstring{$\bfT$}{T}}
\label{app:qr-T}

The decomposition can be performed in two steps: first the QR decomposition $\bfJ_\calW = \bfQ_\calW \bfR_\calW$ of $\bfJ_\calW \defeq \bfJ \bfN_\calW$, followed by the QR decomposition $\sBIN \mu \bfj^T \bfN_\calW \\ \bfR_\calW \sBOUT = \sBIN \bfq_1^T \\ \bfQ_2 \sBOUT \bfR$. Combining these two yields $\bfT$ by:
\begin{align}
  \BIN \mu \bfj^T \bfN_\calW \\ \bfJ \bfN_\calW \BOUT 
	     &= \BIN 1 & 0 \\ 0 & \bfQ_\calW \BOUT \BIN \mu \bfj^T \bfN_\calW \\ \bfR_\calW \BOUT
	      = \BIN 1 & 0 \\ 0 & \bfQ_\calW \BOUT \BIN \bfq_1^T \\ \bfQ_2 \BOUT \bfR \nonumber \\ 
			 &= \BIN \bfq_1^T \\ \bfQ_\calW \bfQ_2 \BOUT \bfR
\end{align}

The matrix $\sBIN \mu \bfj^T \bfN_\calW \\ \bfR_\calW \sBOUT$ is upper Hessenberg, so that its QR decomposition is computed in $O(n^2)$~\cite[Chapter~5]{golub:book:1996}.

The decomposition of $\bfJ_\calW$ can be achieved in $O(n)$ by taking advantage of its structure. To avoid going through several corner cases, we sketch informally how this is done with the help of the example in Figure~\ref{fig:QR} (more details can be found in~\cite{escande:tech:2018}). Because the sum of three non-zero elements on any row (ignoring the first) of $\bfJ$ is zero, a careful study reveals that $\bfJ_\calW$ is made of $p$ tridiagonal blocks, one for each group of consecutive inactive constraints. Blocks $j$ and $j+1$ are separated by $a_j - 1$ rows of zeros, and the last column of the first block is aligned with the first column of the second block (Figure~\ref{fig:QR}, left). We can perform QR decompositions for each blocks separately and denote by $\tilde{\bfQ}_\calW$ the product of all orthogonal matrices. All blocks, except possibly the last one, have a rank equal to their row size minus one, so that the triangular factor of the decomposition has zeros on its last line (Figure~\ref{fig:QR}, middle left). Multiplying by a permutation matrix $\bfP_\calW$, all zero rows can be moved to the bottom, and we get a quasi-tridiagonal matrix (Figure~\ref{fig:QR}, middle right). The latter can be made triangular with a last tridiagonal QR decomposition in $O(n)$ (Figure~\ref{fig:QR}, right).

\begin{figure}
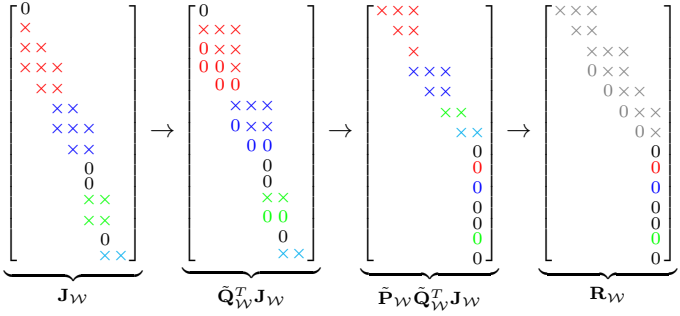

\begin{equation}
\setlength\arraycolsep{0pt}
\resizebox{!}{0.075\textheight}{$
\scriptsize
\nonumber
\underbrace{\BIN
\MATzb \\
\MATrxb  & \\
\MATrxb  & \MATrx \\
\MATrxb  & \MATrx & \MATrx \\
         & \MATrx & \MATrx \\
	     &        & \MATbx & \MATbx \\
		 &        & \MATbx & \MATbx & \MATbx \\
		 &        &        & \MATbx & \MATbx \\
		 &        &        &        & \MATz  \\
		 &        &        &        & \MATz  \\
		 &        &        &        & \MATgx & \MATgx \\
		 &        &        &        & \MATgx & \MATgx \\
		 &        &        &        &        & \MATz  \\
		 &        &        &        &        & \MATcx & \MATcx
\BOUT}_{\bfJ_\calW} 
\rightarrow
\underbrace{\BIN
\MATzb \\
\MATrxb & \MATrx & \MATrx \\
\MATrzb & \MATrx & \MATrx \\
\MATrzb & \MATrz & \MATrx \\
        & \MATrz & \MATrz \\
        &        & \MATbx & \MATbx & \MATbx \\
        &        & \MATbz & \MATbx & \MATbx \\
        &        &        & \MATbz & \MATbz \\
        &        &        &        & \MATz \\
        &        &        &        & \MATz \\
        &        &        &        & \MATgx & \MATgx \\
        &        &        &        & \MATgz & \MATgz \\
        &        &        &        &        & \MATz  \\
        &        &        &        &        & \MATcx & \MATcx
\BOUT}_{\tilde{\bfQ}_\calW^T\bfJ_\calW} 
\rightarrow
\underbrace{\BIN
\MATrxb & \MATrx & \MATrx \\
        & \MATrx & \MATrx \\
        &        & \MATrx \\
        &        & \MATbx & \MATbx & \MATbx \\
        &        &        & \MATbx & \MATbx \\
        &        &        &        & \MATgx & \MATgx \\
        &        &        &        &        & \MATcx & \MATcx \\
        &        &        &        &        &        & \MATz  \\
        &        &        &        &        &        & \MATrz \\
        &        &        &        &        &        & \MATbz \\
        &        &        &        &        &        & \MATz  \\
        &        &        &        &        &        & \MATz  \\
        &        &        &        &        &        & \MATgz \\
        &        &        &        &        &        & \MATz
\BOUT}_{\tilde{\bfP}_\calW\tilde{\bfQ}_\calW^T\bfJ_\calW}
\rightarrow
\underbrace{\BIN
\MATxxb & \MATxx & \MATxx \\
        & \MATxx & \MATxx \\
        &        & \MATxx & \MATxx & \MATxx \\
        &        & \MATxz & \MATxx & \MATxx \\
        &        &        & \MATxz & \MATxx & \MATxx\\
        &        &        &        & \MATxz & \MATxx & \MATxx\\
        &        &        &        &        & \MATxz & \MATxx \\
        &        &        &        &        &        & \MATz  \\
        &        &        &        &        &        & \MATrz \\
        &        &        &        &        &        & \MATbz \\
        &        &        &        &        &        & \MATz  \\
        &        &        &        &        &        & \MATz  \\
        &        &        &        &        &        & \MATgz \\
        &        &        &        &        &        & \MATz
\BOUT}_{\bfR_\calW}
$}
\end{equation}
\caption{\textbf{QR decomposition for $n=15$ and $\calW = \left\{1, 2, 6, 9, 10, 11, 13, 14\right\}$.} Cross symbols $\times$ stand for non-zero elements. Left: block structure of $\bfJ_\calW$, with one color per block. Middle-left: performing QR decomposition for each block. Middle-right: permuting all zero rows to the bottom. Right: completing the QR decomposition.}
\label{fig:QR}
\end{figure}

\subsection{Search for a feasible initial point}
\label{app:initial-point}

Denoting by $\bfl_Z$ and $\bfu_Z$ the bounds corresponding to $\bfC_Z$, the set $\calZ \defeq \left\{\bfphi \in \bbR^n, \bfl_Z \leq \bfC_Z \bfphi \leq \bfu_Z\right\}$ is a zonotope equal to 
\begin{equation}
    \calZ = \bfL_Z \bfl_Z + \bfL \mbox{diag}(\bfu_Z - \bfl_Z) \left[0,1\right]^n
\end{equation}
where $\bfL_Z \defeq \bfC_Z^{-1}$ is the $n \times n$ lower triangular matrix with all coefficients equal to $1$. 
Feasible points for the whole problem are those in $\calZ$ such that $\omegaimin^2 \leq \varphi_n \leq \omegaimax^2$.

Consider the point:
\begin{equation}
    \bfphi_\isubscript(a) \defeq \bfL \bfl_Z + \sum a (u_i-l_i) \bfL_i
\end{equation}
where $\bfL_i$ it the $\th{i}$ column of $\bfL$. This point is in $\calZ$ for any value $a \in \left[0,1\right]$. 
Its last component $\varphi_{\isubscript,n}(a)$ is an increasing linear function $\varphi_{\isubscript,n}(a) = s_l + a s_d$ with $s_l = \sum_j l_j$ and $s_d = \sum_j (u_j-l_j) \geq 0$.
Let us denote by $a^-$ and $a^+$ the two values such that $\varphi_{\isubscript,n}(a^-) = \omegaimin^2$ and $\varphi_{\isubscript,n}(a^+) = \omegaimax^2$. The linear constraints of the capture problem are then feasible \emph{if and only if} $\left[a^-, a^+\right] \cap \left[0, 1\right] \neq \emptyset$. In this case, any $a$ in this intersection yields a feasible point $\bfphi_\isubscript(a)$ for the problem, for instance the middle value $a_m \defeq \frac12 (\max(a^-,0) + \min(a^+,1))$. We can therefore initialize our SQP with:
\begin{align}
    \bfphi_0 & = \bfphi_\isubscript(a_m) & \bflambda_0 & = 0
\end{align}
which is guaranteed to be a feasible point.

\section{External optimization details}

\subsection{Feasibility of outer-optimization problems}
\label{app:feasible-intervals}

There are three ways a bad choice of $\alpha$ can yield an unfeasible capture problem~\eqref{eq:optim-full}:
\begin{enumerate}[(\it a)]
\item The bounds of the inequality~\eqref{eq:omega-i-3d} are such that $\omegaimin > \omegaimax$. Recall that these bounds represent feasibility of the CoP $\bfri$, which involves $\alpha$ by Equation~\eqref{eq:reduc-bxy}.
\item The intersection between the nonlinear equality constraint~\eqref{eq:conv-cons-3d} and the polytope~\eqref{eq:omega-i-3d}--\eqref{eq:optim-full-ineq} is empty. The influence of $\alpha$ on this comes from Equation~\eqref{eq:reduc-bz}.
\item The right cylinder given by the linear constraint~\eqref{eq:omega-i-3d} does not interesect the zonotope~\eqref{eq:optim-full-ineq}. The role of $\alpha$ in this comes once again from its influence on $\bfri$ by Equation~\eqref{eq:reduc-bxy}.
\end{enumerate}
Case (\emph{c}) can be caught efficiently before solving the capture problem (Appendix~\ref{app:initial-point}). While anticipating (\emph{b}) is still an open question, case (\emph{a}) can be avoided altogether thanks to a more careful treatment of CoP inequality constraints.

Recall from Equation~\eqref{eq:ineq-omegai} that:
\begin{equation}
    \left[\alpha \bfH \bfrf^{xy} + (1 - \alpha) \bfp - \bfH \bfci^{xy} \right] \omegai \geq \bfH \bfcdi^{xy}
\end{equation}
Let us rewrite this inequality as:
\begin{equation}
    (\bfu - \alpha \bfv) \omegai \geq \bfw
\end{equation}
To avoid corner cases, let us extend the three vectors $\bfu$, $\bfv$ and $\bfw$ with two additional lines:
\begin{itemize}
    \item \emph{Line 1}: $u = 1$, $v = 0$ and $w = \omegaimin$
    \item \emph{Line 2}: $u = -1$, $v = 0$ and $w = -\omegaimax$
\end{itemize}
Next, note that the two sets $\calA_\textnormal{min}(\alpha) \defeq \{ i, u_i - \alpha v_i \geq 0 \}$ and $\calA_\textnormal{max}(\alpha) \defeq \{ i, u_i - \alpha v_i \leq 0 \}$ are such that:
\begin{align}
    \omegaimin(\alpha) & = \max\left\{\frac{w_i}{u_i - \alpha v_i}, i \in \calA_\textnormal{min}(\alpha)\right\} \\
    \omegaimax(\alpha) & = \min\left\{\frac{w_i}{u_i - \alpha v_i}, i \in \calA_\textnormal{max}(\alpha)\right\}
\end{align}
Using a technique reminiscent of Fourier-Motzkin elimination, a necessary and sufficient condition for $\omegaimin \leq \omegaimax$ is then that, for all pairs $(i, j) \in \calA_\textnormal{min}(\alpha) \times \calA_\textnormal{max}(\alpha)$,
\begin{equation}
    \label{eq:alpha-smart}
    u_i w_j - u_j w_i \leq \alpha (v_i w_j - v_j w_i)
\end{equation}
These inequalities are of the form $\widetilde{\bfu} \alpha \geq \widetilde{\bfv}$ and can therefore be reduced, similarly to Equations~\eqref{eq:calc-omegaimin}--\eqref{eq:calc-omegaimax}, into a single interval $[\alphamin, \alphamax]$. In this interval, it is guaranteed that $\omegaimin \leq \omegaimax$ and failure case (\textit{a}) is avoided.

The subtlety to notice here is that the index sets $\calA_\square(\alpha)$ change when $\alpha$ crosses the roots $u_i / v_i$. (Note that there are few such roots in practice, \emph{e.g.} at most six with rectangular foot soles.) We take this phenomenon into account in the overall Algorithm~\ref{alg:alpha-feas}.

\begin{algorithm}[t]
	\caption{Computation of $\alpha$ feasibility intervals}
    \label{alg:alpha-feas}
    \begin{algorithmic}
        \REQUIRE vectors $\bfu, \bfv$ and $\bfw$
        \ENSURE set $\calI$ of feasible intervals $[\alphamin, \alphamax]$
        \STATE $\calI \leftarrow \emptyset$
        \STATE $\calR \leftarrow \{ r_j = u_j / v_j | r_j \in (0, 1) \}$
        \FOR{$(r_{j}, r_{j+1})$ consecutive roots in \textsc{sort}($\calR$)}
            \STATE $(\alphamin, \alphamax, \alpha) \leftarrow (r_{2j}, r_{2j+1}, \frac12 (r_{2j} + r_{2j + 1}))$
            \STATE {Compute} index sets $\calA_\textnormal{min}(\alpha)$ and $\calA_\textnormal{max}(\alpha)$
            \STATE {Reduce} $[\alphamin, \alphamax]$ using~\eqref{eq:alpha-smart} with $\calA_\textnormal{min}(\alpha)$, $\calA_\textnormal{max}(\alpha)$
            \STATE $\calI \leftarrow \calI \cup \{[\alphamin, \alphamax]\}$
        \ENDFOR
        \STATE \textbf{return} $\calI$
	\end{algorithmic}
\end{algorithm}

\end{document}

%% file: abbr.tex
\newcommand{\dd}[1]{\ensuremath {{\rm d}{#1}}}
\newcommand{\defeq}{:=}

\renewcommand{\th}[1]{\ensuremath {#1}^{\textrm{th}}}

\def\bfcd{\dot{\bfc}}

\def\bfrd{\dot{\bfr}}
\def\bfxd{\dot{\bfx}}
\def\bfxid{\dot{\bfxi}}
\def\bfzd{\dot{\bfz}}
\def\bfzetad{\dot{\bfzeta}}
\def\gammad{\dot{\gamma}}
\def\omegad{\dot{\omega}}

\def\sd{\dot{s}}

\def\bfcdd{\ddot{\bfc}}

\def\sdd{\ddot{s}}

\newcommand{\bfzeta}{\boldsymbol{\zeta}}

\newcommand{\bflambda}{\boldsymbol{\lambda}}

\newcommand{\bfxi}{\boldsymbol{\xi}}

\newcommand{\bftau}{\boldsymbol{\tau}}

\newcommand{\bfphi}{\boldsymbol{\phi}}

\newcommand{\bfvarphi}{\boldsymbol{\varphi}}

\newcommand{\bfb}{\ensuremath {\bm{b}}}
\newcommand{\bfc}{\ensuremath {\bm{c}}}
\newcommand{\bfd}{\ensuremath {\bm{d}}}
\newcommand{\bfe}{\ensuremath {\bm{e}}}
\newcommand{\bff}{\ensuremath {\bm{f}}}
\newcommand{\bfg}{\ensuremath {\bm{g}}}
\newcommand{\bfh}{\ensuremath {\bm{h}}}

\newcommand{\bfj}{\ensuremath {\bm{j}}}

\newcommand{\bfl}{\ensuremath {\bm{l}}}

\newcommand{\bfn}{\ensuremath {\bm{n}}}
\newcommand{\bfo}{\ensuremath {\bm{o}}}
\newcommand{\bfp}{\ensuremath {\bm{p}}}
\newcommand{\bfq}{\ensuremath {\bm{q}}}
\newcommand{\bfr}{\ensuremath {\bm{r}}}

\newcommand{\bft}{\ensuremath {\bm{t}}}
\newcommand{\bfu}{\ensuremath {\bm{u}}}
\newcommand{\bfv}{\ensuremath {\bm{v}}}
\newcommand{\bfw}{\ensuremath {\bm{w}}}
\newcommand{\bfx}{\ensuremath {\bm{x}}}

\newcommand{\bfz}{\ensuremath {\bm{z}}}

\newcommand{\bfA}{\mathbf{A}}
\newcommand{\bfB}{\mathbf{B}}
\newcommand{\bfC}{\mathbf{C}}

\newcommand{\bfH}{\mathbf{H}}
\newcommand{\bfI}{\mathbf{I}}
\newcommand{\bfJ}{\mathbf{J}}

\newcommand{\bfL}{\mathbf{L}}
\newcommand{\bfM}{\mathbf{M}}
\newcommand{\bfN}{\mathbf{N}}

\newcommand{\bfP}{\mathbf{P}}
\newcommand{\bfQ}{\mathbf{Q}}
\newcommand{\bfR}{\mathbf{R}}
\newcommand{\bfS}{\mathbf{S}}
\newcommand{\bfT}{\mathbf{T}}

\newcommand{\calA}{{\cal A}}

\newcommand{\calC}{{\cal C}}

\newcommand{\calE}{{\cal E}}

\newcommand{\calI}{{\cal I}}

\newcommand{\calL}{{\cal L}}

\newcommand{\calR}{{\cal R}}
\newcommand{\calS}{{\cal S}}

\newcommand{\calW}{{\cal W}}

\newcommand{\calZ}{{\cal Z}}

\newcommand{\bbR}{{\mathbb{R}}}

%% file: capture-walking.bbl
\begin{thebibliography}{10}
\providecommand{\url}[1]{#1}
\csname url@samestyle\endcsname
\providecommand{\newblock}{\relax}
\providecommand{\bibinfo}[2]{#2}
\providecommand{\BIBentrySTDinterwordspacing}{\spaceskip=0pt\relax}
\providecommand{\BIBentryALTinterwordstretchfactor}{4}
\providecommand{\BIBentryALTinterwordspacing}{\spaceskip=\fontdimen2\font plus
\BIBentryALTinterwordstretchfactor\fontdimen3\font minus
  \fontdimen4\font\relax}
\providecommand{\BIBforeignlanguage}[2]{{%
\expandafter\ifx\csname l@#1\endcsname\relax
\typeout{** WARNING: IEEEtran.bst: No hyphenation pattern has been}%
\typeout{** loaded for the language `#1'. Using the pattern for}%
\typeout{** the default language instead.}%
\else
\language=\csname l@#1\endcsname
\fi
#2}}
\providecommand{\BIBdecl}{\relax}
\BIBdecl

\bibitem{pratt2006humanoids}
J.~Pratt, J.~Carff, S.~Drakunov, and A.~Goswami, ``Capture point: A step toward
  humanoid push recovery,'' in \emph{IEEE-RAS International Conference on
  Humanoid Robots}, 2006, pp. 200--207.

\bibitem{koolen2012ijrr}
T.~Koolen, T.~de~Boer, J.~Rebula, A.~Goswami, and J.~Pratt,
  ``Capturability-based analysis and control of legged locomotion, part 1:
  Theory and application to three simple gait models,'' \emph{The International
  Journal of Robotics Research}, vol.~31, no.~9, pp. 1094--1113, 2012.

\bibitem{sugihara2009icra}
T.~Sugihara, ``Standing stabilizability and stepping maneuver in planar
  bipedalism based on the best com-zmp regulator,'' in \emph{IEEE International
  Conference on Robotics and Automation}, 2009, pp. 1966--1971.

\bibitem{takenaka2009iros}
T.~Takenaka, T.~Matsumoto, and T.~Yoshiike, ``Real time motion generation and
  control for biped robot-1st report: Walking gait pattern generation,'' in
  \emph{IEEE/RSJ International Conference on Intelligent Robots and Systems},
  2009, pp. 1084--1091.

\bibitem{morisawa2012humanoids}
M.~Morisawa, S.~Kajita, F.~Kanehiro, K.~Kaneko, K.~Miura, and K.~Yokoi,
  ``Balance control based on capture point error compensation for biped walking
  on uneven terrain,'' in \emph{IEEE-RAS International Conference on Humanoid
  Robots}, 2012, pp. 734--740.

\bibitem{englsberger2015tro}
J.~Englsberger, C.~Ott, and A.~Albu-Sch{\"a}ffer, ``Three-dimensional bipedal
  walking control based on divergent component of motion,'' \emph{IEEE
  Transactions on Robotics}, vol.~31, no.~2, pp. 355--368, 2015.

\bibitem{griffin2017iros}
R.~J. Griffin, G.~Wiedebach, S.~Bertrand, A.~Leonessa, and J.~Pratt, ``Walking
  stabilization using step timing and location adjustment on the humanoid
  robot, atlas,'' in \emph{IEEE/RSJ International Conference on Intelligent
  Robots and Systems}, 2017.

\bibitem{pratt2007icra}
J.~E. Pratt and S.~V. Drakunov, ``Derivation and application of a conserved
  orbital energy for the inverted pendulum bipedal walking model,'' in
  \emph{IEEE International Conference on Robotics and Automation}, 2007, pp.
  4653--4660.

\bibitem{ramos2015humanoids}
O.~E. Ramos and K.~Hauser, ``Generalizations of the capture point to nonlinear
  center of mass paths and uneven terrain,'' in \emph{IEEE-RAS International
  Conference on Humanoid Robots}, 2015, pp. 851--858.

\bibitem{koolen2016humanoids}
T.~Koolen, M.~Posa, and R.~Tedrake, ``Balance control using center of mass
  height variation: Limitations imposed by unilateral contact,'' in
  \emph{IEEE-RAS International Conference on Humanoid Robots}, 2016.

\bibitem{caron2018icra}
S.~Caron and B.~Mallein, ``Balance control using both zmp and com height
  variations: A convex boundedness approach,'' in \emph{IEEE International
  Conference on Robotics and Automation}, May 2018.

\bibitem{kajita2001iros}
S.~Kajita, F.~Kanehiro, K.~Kaneko, K.~Yokoi, and H.~Hirukawa, ``The 3d linear
  inverted pendulum mode: A simple modeling for a biped walking pattern
  generation,'' in \emph{{IEEE}/{RSJ} International Conference on Intelligent
  Robots and Systems}, vol.~1, 2001, pp. 239--246.

\bibitem{caron2016tro}
S.~Caron, Q.-C. Pham, and Y.~Nakamura, ``Zmp support areas for multi-contact
  mobility under frictional constraints,'' \emph{IEEE Transactions on
  Robotics}, vol.~33, no.~1, pp. 67--80, Feb. 2017.

\bibitem{morisawa2005icra}
M.~Morisawa, S.~Kajita, K.~Kaneko, K.~Harada, F.~Kanehiro, K.~Fujiwara, and
  H.~Hirukawa, ``Pattern generation of biped walking constrained on parametric
  surface,'' in \emph{{IEEE} International Conference on Robotics and
  Automation}, 2005, pp. 2405--2410.

\bibitem{zhao2012humanoids}
Y.~Zhao and L.~Sentis, ``A three dimensional foot placement planner for
  locomotion in very rough terrains,'' in \emph{Humanoid Robots (Humanoids),
  2012 12th {IEEE}-{RAS} International Conference on}, 2012, pp. 726--733.

\bibitem{hopkins2014humanoids}
M.~A. Hopkins, D.~W. Hong, and A.~Leonessa, ``Humanoid locomotion on uneven
  terrain using the time-varying divergent component of motion,'' in
  \emph{IEEE-RAS International Conference on Humanoid Robots}, 2014, pp.
  266--272.

\bibitem{kamioka2015iros}
T.~Kamioka, T.~Watabe, M.~Kanazawa, H.~Kaneko, and T.~Yoshiike, ``Dynamic gait
  transition between bipedal and quadrupedal locomotion,'' in
  \emph{{IEEE}/{RSJ} International Conf. on Int. Robots and Systems}, 2015.

\bibitem{caron2017iros}
S.~Caron and A.~Kheddar, ``Dynamic walking over rough terrains by nonlinear
  predictive control of the floating-base inverted pendulum,'' in
  \emph{IEEE/RSJ International Conf. on Intelligent Robots and Systems}, 2017.

\bibitem{hauser2004cdc}
J.~Hauser, A.~Saccon, and R.~Frezza, ``Achievable motorcycle trajectories,'' in
  \emph{IEEE Conf. on Dec. and Control}, vol.~4, 2004, pp. 3944--3949.

\bibitem{coppel1971sdeds}
W.~A. Coppel, ``Dichotomies and stability theory,'' in \emph{Symposium on
  Differential Equations and Dynamical Systems}, 1971, pp. 160--162.

\bibitem{lanari2014humanoids}
L.~Lanari, S.~Hutchinson, and L.~Marchionni, ``Boundedness issues in planning
  of locomotion trajectories for biped robots,'' in \emph{IEEE-RAS
  International Conference on Humanoid Robots}, 2014, pp. 951--958.

\bibitem{stephens2007humanoids}
B.~Stephens, ``Humanoid push recovery,'' in \emph{IEEE-RAS International
  Conference on Humanoid Robots}, 2007, pp. 589--595.

\bibitem{yamamoto2016ras}
K.~Yamamoto, ``Control strategy switching for humanoid robots based on maximal
  output admissible set,'' \emph{Robotics and Autonomous Systems}, vol.~81, pp.
  17--32, Jul. 2016.

\bibitem{samy2017humanoids}
V.~Samy, S.~Caron, K.~Bouyarmane, and A.~Kheddar, ``Post-impact adaptive
  compliance for humanoid falls using predictive control of a reduced model,''
  in \emph{IEEE-RAS Int. Conf. on Humanoid Robots}, 2017.

\bibitem{delprete2017hal}
A.~Del~Prete, S.~Tonneau, and N.~Mansard, ``{Zero Step Capturability for Legged
  Robots in Multi Contact},'' Dec. 2017, submitted.

\bibitem{pham2018tro}
H.~Pham and Q.-C. Pham, ``A new approach to time-optimal path parameterization
  based on reachability analysis,'' \emph{IEEE Transactions on Robotics},
  vol.~34, no.~3, pp. 645--659, Jun. 2018.

\bibitem{nocedal:book:2006}
J.~Nocedal and S.~J. Wright, \emph{Numerical Optimization}, 2nd~ed.\hskip 1em
  plus 0.5em minus 0.4em\relax Springer-Verlag New York, 2006.

\bibitem{golub:book:1996}
G.~H. Golub and C.~F. Van~Loan, \emph{Matrix Computations (3rd Ed.)}.\hskip 1em
  plus 0.5em minus 0.4em\relax Baltimore, MD, USA: Johns Hopkins University
  Press, 1996.

\bibitem{bjorck:book:1996}
{\AA}.~Bj{\"o}rck, \emph{Numerical Methods for Least Squares Problems}.\hskip
  1em plus 0.5em minus 0.4em\relax Society for Industrial and Applied
  Mathematics, 1996.

\bibitem{wachter2006springer}
A.~W{\"a}chter and L.~T. Biegler, ``On the implementation of an interior-point
  filter line-search algorithm for large-scale nonlinear programming,''
  \emph{Mathematical Programming}, vol. 106, no.~1, pp. 25--57, Mar. 2006.

\bibitem{gill:tech:1986}
P.~E.~E. Gill, S.~J. Hammarling, W.~Murray, M.~A. Saunders, and M.~H. Wright,
  ``User's guide for lssol (version 1.0),'' Standford University, Standord,
  California 94305, Tech. Rep. 86-1, January 1986.

\bibitem{khadiv2016humanoids}
M.~Khadiv, A.~Herzog, S.~A.~A. Moosavian, and L.~Righetti, ``Step timing
  adjustment: A step toward generating robust gaits,'' in \emph{IEEE-RAS
  International Conference on Humanoid Robots}, 2016, pp. 35--42.

\bibitem{sugihara2017iros}
T.~Sugihara and T.~Yamamoto, ``Foot-guided agile control of a biped robot
  through zmp manipulation,'' in \emph{Intelligent Robots and Systems, IEEE/RSJ
  International Conference on}, 2017.

\bibitem{escande2014ijrr}
A.~Escande, N.~Mansard, and P.-B. Wieber, ``Hierarchical quadratic programming:
  Fast online humanoid-robot motion generation,'' \emph{The Int. Journal of
  Robotics Research}, vol.~33, no.~7, pp. 1006--1028, 2014.

\bibitem{wieber2006humanoids}
P.-B. Wieber, ``Trajectory free linear model predictive control for stable
  walking in the presence of strong perturbations,'' in \emph{IEEE-RAS
  International Conference on Humanoid Robots}, 2006, pp. 137--142.

\bibitem{brasseur2015humanoids}
C.~Brasseur, A.~Sherikov, C.~Collette, D.~Dimitrov, and P.-B. Wieber, ``A
  robust linear {MPC} approach to online generation of {3D} biped walking
  motion,'' in \emph{IEEE--RAS Int. Conf. on Humanoid Robots}, 2015.

\bibitem{dai2014humanoids}
H.~Dai, A.~Valenzuela, and R.~Tedrake, ``Whole-body motion planning with
  centroidal dynamics and full kinematics,'' in \emph{IEEE-RAS International
  Conference on Humanoid Robots}, 2014, pp. 295--302.

\bibitem{carpentier2016icra}
J.~Carpentier, S.~Tonneau, M.~Naveau, O.~Stasse, and N.~Mansard, ``A versatile
  and efficient pattern generator for generalized legged locomotion,'' in
  \emph{IEEE International Conference on Robotics and Automation}, 2016, pp.
  3555--3561.

\bibitem{caron2018hal}
\BIBentryALTinterwordspacing
S.~Caron, A.~Kheddar, and O.~Tempier, ``Stair climbing stabilization of the
  {HRP}-4 humanoid robot using whole-body admittance control,'' Sep. 2018,
  submitted. [Online]. Available:
  \url{https://hal.archives-ouvertes.fr/hal-01875387}
\BIBentrySTDinterwordspacing

\bibitem{kajita2003icra}
S.~Kajita, F.~Kanehiro, K.~Kaneko, K.~Fujiwara, K.~Harada, K.~Yokoi, and
  H.~Hirukawa, ``Biped walking pattern generation by using preview control of
  zero-moment point,'' in \emph{IEEE International Conference on Robotics and
  Automation}, vol.~2, 2003, pp. 1620--1626.

\bibitem{nakaoka2007rsj}
S.~Nakaoka, S.~Hattori, F.~Kanehiro, S.~Kajita, and H.~Hirukawa, ``Iterative
  contact force solver for simulating articulated rigid bodies,'' in
  \emph{Annual conference of the Robotics Society of Japan}, 2007, in Japanese.

\bibitem{terada2007humanoids}
K.~Terada and Y.~Kuniyoshi, ``Online gait planning with dynamical
  3d-symmetrization method,'' in \emph{IEEE-RAS International Conference on
  Humanoid Robots}, 2007, pp. 222--227.

\bibitem{herdt2012humanoids}
A.~Herdt, N.~Perrin, and P.-B. Wieber, ``Lmpc based online generation of more
  efficient walking motions,'' in \emph{IEEE-RAS International Conference on
  Humanoid Robots}, 2012, pp. 390--395.

\bibitem{scianca2016humanoids}
N.~Scianca, M.~Cognetti, D.~De~Simone, L.~Lanari, and G.~Oriolo,
  ``Intrinsically stable {MPC} for humanoid gait generation,'' in
  \emph{IEEE-RAS International Conference on Humanoid Robots}, 2016, pp.
  101--108.

\bibitem{lanari2015humanoids}
L.~Lanari and S.~Hutchinson, ``Planning desired center of mass and zero moment
  point trajectories for bipedal locomotion,'' in \emph{IEEE-RAS International
  Conference on Humanoid Robots}, 2015, pp. 637--642.

\bibitem{sugihara2009tro}
T.~Sugihara and Y.~Nakamura, ``Boundary condition relaxation method for
  stepwise pedipulation planning of biped robots,'' \emph{{IEEE} Transactions
  on Robotics}, vol.~25, no.~3, pp. 658--669, 2009.

\bibitem{caron2016humanoids}
S.~Caron and A.~Kheddar, ``Multi-contact walking pattern generation based on
  model preview control of 3d com accelerations,'' in \emph{IEEE-RAS
  International Conference on Humanoid Robots}, 2016.

\bibitem{kajita2017humanoids}
S.~Kajita, M.~Benallegue, R.~Cisneros, T.~Sakaguchi, S.~Nakaoka, M.~Morisawa,
  K.~Kaneko, and F.~Kanehiro, ``Biped walking pattern generation based on
  spatially quantized dynamics,'' in \emph{IEEE-RAS International Conference on
  Humanoid Robots}, 2017, pp. 599--605.

\bibitem{brossette:phd:2016}
S.~Brossette, ``Viable multi-contact posture computation for humanoid robots
  using nonlinear optimization on manifolds,'' Ph.D. dissertation, University
  of Montpellier, 2016.

\bibitem{escande:tech:2018}
\BIBentryALTinterwordspacing
A.~Escande, ``Dedicated optimization for {MPC} with convex boundedness
  constraint,'' CNRS-AIST JRL UMI3218/RL, Tech. Rep., 2018. [Online].
  Available: \url{https://github.com/jrl-umi3218/CaptureProblemSolver/}
\BIBentrySTDinterwordspacing

\end{thebibliography}
